%% file: main.tex
\documentclass{article}

\PassOptionsToPackage{numbers, compress}{natbib}

\usepackage{pdt_report}

\input{commands}
\definecolor{boilermakergold}{HTML}{cfb991}
\definecolor{steel}{HTML}{555960}
\definecolor{coolgray}{HTML}{6f727b}
\definecolor{dust}{HTML}{EBD99F}
\usepackage{cleveref}
\crefname{corollary}{corollary}{corollaries}
\Crefname{corollary}{Corollary}{Corollaries}

\title{On Variance Reduction in Learning Mean Flows}

\author{%
  Juanwu~Lu \\
  Purdue University \\
  \texttt{juanwu@purdue.edu}
  \And
  Ziran Wang \\
  Purdue University \\
  \texttt{ziran@purdue.edu}
}

\begin{document}

\maketitle

% -----------------------------------------
% Main contents (separate bibliography unit)
\begin{bibunit}[unsrt]
\input{contents/abstract}
\input{contents/introduction}
\input{contents/method}
\input{contents/related}
\input{contents/experiment}
\input{contents/conclusion}
\input{contents/postscripts}
{%
    \small
    \putbib[references]
}
\end{bibunit}

% Comment this out if using `pdt_report' style
% \newpage
% \input{contents/checklist}

% -----------------------------------------
% Appendix (separate bibliography unit)
\begin{bibunit}[unsrt]
\input{contents/supplementary}
\end{bibunit}

\end{document}

%% file: commands.tex
\usepackage{bibunits}
\defaultbibliographystyle{unsrt}
\defaultbibliography{references}

\usepackage{algorithm}
\usepackage{algpseudocode}
\usepackage{amsmath}
\usepackage{amssymb}
\usepackage{amsthm}
\usepackage[toc,page]{appendix}
\usepackage{array}
\usepackage{booktabs} % for professional tables
\usepackage{bm}
\usepackage{cancel}
\usepackage{graphicx}
\usepackage{lipsum}
\usepackage{mathtools}
\usepackage[framemethod=default]{mdframed}
\usepackage{multicol}
\usepackage{subcaption}
\usepackage[most]{tcolorbox}
\usepackage{titletoc}
\usepackage{tikz}
\usepackage{url}
\usepackage{wrapfig}
\usepackage[dvipsnames]{xcolor}
\usepackage[textsize=tiny]{todonotes}

\usetikzlibrary{bayesnet}

\theoremstyle{plain}
\newtheorem{theorem}{Theorem}
\newtheorem{proposition}{Proposition}
\newtheorem{lemma}[theorem]{Lemma}

\theoremstyle{definition}

\theoremstyle{remark}

\definecolor{softteal}{RGB}{70, 190, 180} % Adjust RGB for your preferred shade
\usepackage{microtype}

\def\eqref#1{equation~\ref{#1}}
\def\1{\bm{1}}

\def\rvx{{\mathbf{x}}}

\def\vtheta{{\bm{\theta}}}

\def\vb{{\bm{b}}}

\def\vf{{\bm{f}}}
\def\vg{{\bm{g}}}

\def\vr{{\bm{r}}}

\def\vu{{\bm{u}}}
\def\vv{{\bm{v}}}

\def\vx{{\bm{x}}}

\def\mA{{\bm{A}}}

\def\mG{{\bm{G}}}

\def\mI{{\bm{I}}}
\def\mJ{{\bm{J}}}

\def\mM{{\bm{M}}}

\DeclareMathAlphabet{\mathsfit}{\encodingdefault}{\sfdefault}{m}{sl}
\SetMathAlphabet{\mathsfit}{bold}{\encodingdefault}{\sfdefault}{bx}{n}

\def\gD{{\mathcal{D}}}

\def\gL{{\mathcal{L}}}

\def\gN{{\mathcal{N}}}
\def\gO{{\mathcal{O}}}

\def\gU{{\mathcal{U}}}

\def\gX{{\mathcal{X}}}

\newcommand{\pdata}{p_{\rm{data}}}
\newcommand{\E}{\mathbb{E}}

\newcommand{\R}{\mathbb{R}}

\newcommand{\Var}{\mathrm{Var}}

\newcommand{\Cov}{\mathrm{Cov}}
\usepackage{makecell}

\DeclareMathOperator{\Tr}{Tr}

\makeatletter
\newcommand{\namedlabel}[2]{\begingroup
    \def\@currentlabel{#2}%
    \label{#1}\textbf{#2}%
\endgroup}
\makeatother

%% file: contents/abstract.tex
\begin{abstract}
    \label{sec: abstract}
    One-step generative modeling has emerged as a leading approach for amortizing the inference cost of diffusion and flow-matching models. Among distillation-free methods, \emph{MeanFlow} training is notoriously unstable, with non-decreasing loss and unbounded gradient variance. In this work, we establish a theory that attributes this pathology to a misuse of the conditional velocity field. We show that the conditional velocity plays two distinct statistical roles in the loss: both as an unbiased regression target and as a Monte Carlo control variate in a Jacobi-vector product, with the original MeanFlow loss assigning the wrong coefficient to the latter. We derive the optimal coefficient in closed form and show that a family of fixes in concurrent works corresponds to different practical realizations of the same optimum. A controlled sweep of this coefficient on two-dimensional benchmarks and on a latent Diffusion Transformer recovers the predicted bias-variance ordering. Our DiT experiment also reveals a quantitative \emph{FID-MSE landscape mismatch}. Specifically, although the gradient-MSE is minimized at an interior coefficient value near $\beta\!=\!0.94$, the coefficient that minimizes FID prefers to use conditional velocity directly at the unbiased corner. Our analysis therefore explains \emph{why} MeanFlow is unstable and unifies its concurrent remedies, and shows that the variance-optimal coefficient need not coincide with the quality-optimal one.
\end{abstract}


%% file: contents/introduction.tex
\section{Introduction}
\label{sec: introduction}

%%%%%%%%%%%%%%%%%%%%%%%%%%%%%%%%%%%%%%%%%
% paragraph 1: background
Deep generative models that leverage neural networks to parametrize and sample from unknown data distributions have achieved considerable success~\cite{ho2020denoising,song2021scorebased,karras2022elucidating,dhariwal2021diffusion,esser2024scaling}. Among these models, flow-based generative models~\cite{chen2018neural,rezende2015variational,dinh2017density,grathwohl2019ffjord,lipman2023flowmatchinggenerativemodeling,tong2024improvinggeneralizingflowbasedgenerative,albergo2023buildingnormalizingflowsstochastic} learn a continuous-time transformation between a simple prior and the data distribution. Conditional flow matching~\cite{lipman2023flowmatchinggenerativemodeling,tong2024improvinggeneralizingflowbasedgenerative} and stochastic interpolants~\cite{albergo2023buildingnormalizingflowsstochastic,albergo2025stochasticinterpolants} further enable simulation-free training by regressing a velocity field, achieving sample quality competitive with diffusion models~\cite{ho2020denoising,song2019generative,song2021scorebased}. However, integration at inference time remains a challenge due to its computational costs and numerical instability, motivating a series of subsequent works on \emph{one-step generative models}, including progressive distillation~\cite{salimans2022progressivedistillationfastsampling}, distribution matching~\cite{yin2024distributionmatchingdistillation,yin2024improveddistributionmatchingdistillation}, consistency models~\cite{song2023consistencymodels,lu2025simplifyingstabilizingscalingcontinuoustime}, and flow map matching~\cite{boffi2025flowmapmatchingstochastic}. Nevertheless, most of them require a pre-trained distillation teacher.

%%%%%%%%%%%%%%%%%%%%%%%%%%%%%%%%%%%%%%%%%
% paragraph 2: meanflow and its pathology
Mean Flows~\cite{geng2025meanflowsonestepgenerative} provide a \emph{distillation-free} alternative by learning a two-parameter average velocity field through least-squares regression of a total-derivative identity between the average and instantaneous velocity. In principle, this identity yields exact self-supervision. In practice, however, training is plagued by \emph{non-decreasing loss} and prohibitively \emph{high gradient variance}~\cite{zhang2025alphaflowunderstandingimprovingmeanflow,geng2025improvedmeanflowschallenges}. A plethora of concurrent works have proposed remedies that target different mechanisms. AlphaFlow~\cite{zhang2025alphaflowunderstandingimprovingmeanflow} attributes slow convergence to gradient conflict between flow-matching and total-derivative consistency components and proposes an $\alpha$-curriculum. Improved MeanFlow~\cite{geng2025improvedmeanflowschallenges} replaces the conditional-velocity tangent with a learned marginal-velocity head. Re-MeanFlow~\cite{zhang2026overcomingcurvaturebottleneckmeanflow} preprocesses the data with rectified-flow straightening to reduce trajectory curvature. Terminal Velocity Matching~\cite{zhou2026terminalvelocitymatching} sidesteps the spatial Jacobian by differentiating with respect to the terminal time. Functional Mean Flows~\cite{li2025functionalmeanflowhilbert} extend MeanFlow to Hilbert spaces and study marginal-conditional consistency for functional data. Each fix is empirically effective in isolation, but no theory explains \emph{why} the original objective is unstable or \emph{what} these fixes have in common.

%%%%%%%%%%%%%%%%%%%%%%%%%%%%%%%%%%%%%%%%%
% paragraph 3: the research question + the framing answer
In this paper, we establish our theory around a single research question:
\begin{mdframed}[%
    leftline=true, linecolor=BrickRed, linewidth=2pt, topline=false,
    rightline=false, bottomline=false, backgroundcolor=dust!15
]
\noindent\textbf{Why do stochastic tangents destabilize learning Mean Flows?}
\end{mdframed}
\noindent We answer this question through the lens of \emph{variance reduction}. Specifically, we identify two statistically distinct roles for the conditional velocity in the MeanFlow loss. As the regression target, it is an \textit{unbiased} estimator of the marginal velocity. As the tangent inside the Jacobian-vector product (JVP) of the total-derivative identity, however, it acts as a Monte Carlo \textit{control variate} for the inaccessible marginal velocity, with a statistically suboptimal coefficient. The bias-variance trade-off induced by this substitution governs the variance of the gradient through a quadratic term. A Jacobi factor in JVP amplifies sample-based conditional fluctuations at each gradient step, and stop-gradient further hides this amplification from the optimizer, leading to the empirical pathology.

%%%%%%%%%%%%%%%%%%%%%%%%%%%%%%%%%%%%%%%%%
% paragraph 4: contributions
Following the theory, we derive the optimal tangent control-variate coefficient in closed form and show that it converges to a deterministic-tangent estimator in the practical regime where a marginal-velocity estimator is available. This result \textit{explains} why deterministic tangents (\emph{e.g., learned velocity heads}) help stabilize training. They are essentially different practical instantiations of the same statistical optimum. To empirically validate the framework, we perform a controlled sweep of the coefficient across two-dimensional datasets and the ImageNet dataset using latent diffusion Transformers (DiTs). Direct measurement of per-step gradient variance recovers the predicted non-decreasing loss component. Meanwhile, we observe a $1.2\!\times\!\!\sim\!\!4.3\!\times$ reduction with deterministic tangents on the two-dimensional datasets. The same sweep on ImageNet also reveals a quantitative \emph{FID-MSE landscape mismatch}: the FID ordering aligns with the bias-variance prediction across all four coefficient values. Nonetheless, the FID-axis offset ratio across large coefficients is super-linear in the MSE-axis, hence the FID-optimal coefficient shifts past the gradient-MSE interior minimizer to the \emph{unbiased} corner. Our contributions in this paper are as follows:
\begin{description}
    \item[\namedlabel{cb: thm}{C1}] We establish a theory to prove that per-step gradient variance in the original MeanFlow scales unboundedly with the spatial-Jacobi factor in JVP and the stop-gradient operator induces a \emph{semi-gradient gap} that hides this term from the optimizer.
    \item[\namedlabel{cb: cv}{C2}] We frame the JVP tangent as a \emph{control variate} for the marginal velocity, derive the optimal coefficient in closed form, and unify concurrent works by showing that they each correspond to a different practical realization of the optimal coefficient.
    \item[\namedlabel{cb: validate}{C3}] We empirically validate the framework through a controlled sweep of the coefficient on two-dimensional datasets, dense Gaussian mixtures, and ImageNet with latent DiTs. We observe $1.2\!\times\!\!\sim\!\!4.3\!\times$ direct gradient-variance reduction, sample-quality gains on low-dimensional benchmarks where the deterministic-tangent proxy is accurate (up to $54\%$ SW$_{1}$ on \texttt{swiss\_roll}), and a converged FID ordering across the four coefficient values consistent with the bias-variance prediction. The same DiT measurement also reveals a quantitative FID-MSE landscape mismatch: the coefficient that minimizes gradient-MSE is near $0.94$, whereas the one that minimizes FID leans toward the unbiased corner $0$.
\end{description}

%% file: contents/method.tex
%%%%%%%%%%%%%%%%%%%%%%%%%%%%%%%%%%%%%%%%%
% Preliminary
%
% In this section, I will briefly discuss conditional flow-matching and Mean Flows for one-step generation
%
\section{Preliminaries}
\label{sec: preliminary}

\noindent\textbf{Flow-Matching.} Let $\vx$ denote a data sample in $\gX\subset\mathbb{R}^{d}$ . Flow-matching generative models~\cite{lipman2023flowmatchinggenerativemodeling,tong2024improvinggeneralizingflowbasedgenerative} learn a time-dependent diffeomorphism $\psi(\vx,t):\mathbb{R}^{d}\!\times\![0,1]\to\mathbb{R}^{d}$ from a simple distribution $\rvx_{1}\!\sim\!{p}(\rvx,1)$ to the target $\rvx_{0}\!\sim\!{p}(\rvx,0)\!=\!\pdata$ by parameterizing a velocity field $\vv(\vx,t)\!\triangleq\!\frac{\mathrm{d}}{\mathrm{d}t}\psi(\vx,t)$. Since the exact \textit{marginal velocity field} $\vv(\vx,t)$ is generally inaccessible, conditional flow-matching~\cite{lipman2023flowmatchinggenerativemodeling} instead learns a conditional field $\vv_{\text{cond}}\!=\!\vx_{1}-\vx_{0}$ with $\vx_{t}\!=\!(1-t)\vx_{0}+t\vx_{1}$. The following lemma justifies the substitution. See~\Cref{appx: first} for its proof.
\begin{lemma}
    Given vector fields $\vv_{\text{cond}}$ generating conditional probability paths $p(\vx,t\mid\vx_{0})$, for any $\rvx_{0}\!\sim\!p(\rvx_{0})$, the expected conditional velocity field is equal to the marginal velocity field
    \begin{equation}
        \vv(\vx,t)=\E_{\vx_{0}\sim{p(\vx_{0}\mid\vx_{t}=\vx)}}\!\left[\vv(\vx,t\mid\vx_{0})\right].
        \label{eq: margv-condv-relation}
    \end{equation}
\end{lemma}

%%%%%%%%%%%%%%%%%%%%%%%%%%%%%%%%%%%%%%%%%
% MeanFlow
%%%%%%%%%%%%%%%%%%%%%%%%%%%%%%%%%%%%%%%%%
\noindent\textbf{One-step Mean Flows.} Iterative integration of $\vv(\vx,t)$ at inference is expensive and can be numerically unstable. To bypass it, Mean Flows~\cite{geng2025meanflowsonestepgenerative} learn a two-parameter average velocity field $\vu(\rvx,r,t)$ through regressing an identity along the flow trajectory $\vx_{\tau}\!=\!\psi(\vx_{r},\tau)$ with terminal at $\vx_{t}\!=\!\vx$:
\begin{equation}
    (t-r)\,\vu(\vx,r,t)=\int_{r}^{t}\vv(\vx_{\tau},\tau)\,\mathrm{d}\tau.
    \label{eq: meanflow-identity}
\end{equation}
Differentiating with respect to the terminal time $t$ along the same trajectory yields
\begin{equation}
    \vu(\vx,r,t)+(t-r)\tfrac{\mathrm{d}}{\mathrm{d}t}\vu(\vx,r,t)=\vv(\vx,t),\qquad
    \tfrac{\mathrm{d}}{\mathrm{d}t}\vu=\partial_{\vx}\vu\!\cdot\!\vv(\vx,t)+\partial_{t}\vu,
    \label{eq: td-meanflow-identity}
\end{equation}
where the total derivative $\frac{\mathrm{d}}{\mathrm{d}t}\vu$ is essentially a JVP, denoted by $\texttt{JVP}(\vu,(\vx,r,t),(\vv(\vx,t),0,1))$. The original MeanFlow paper~\cite{geng2025meanflowsonestepgenerative} replaces \emph{both} occurrences of the marginal field with the conditional field $\vv_{\text{cond}}\!=\!\vx_{1}\!-\!\vx_{0}$, giving the MeanFlow loss
\begin{equation}
    \mathcal{L}_{\text{MF}}(\vtheta)=
    \E_{r,t,\vx_{0},\vx_{1}}\!\left[\left\|\vu_{\vtheta}+(t-r)\,\texttt{sg}\!\left[\texttt{JVP}(\vu_{\vtheta},(\vx_{t},r,t),(\vv_{\text{cond}},0,1))\right]-\vv_{\text{cond}}\right\|_{2}^{2}\right],
    \label{eq: mf-loss}
\end{equation}
with $r,t\!\sim\!\gU(0,1)$, $r\!\leq\!t$, $\vx_{0}\!\sim\!\pdata$, $\vx_{1}\!\sim\!\gN(0,\mI_{d})$, $\vx_{t}\!=\!(1-t)\vx_{0}+t\vx_{1}$, and the stop-gradient $\texttt{sg}[\cdot]$ avoiding double backpropagation through the JVP.

%%%%%%%%%%%%%%%%%%%%%%%%%%%%%%%%%%%%%%%%%%%%%%%%%%%%%%%%%%%%%%%%%%%%%%%%%%%%%%%%
% Variance Reduction in Learning Mean Flows
%%%%%%%%%%%%%%%%%%%%%%%%%%%%%%%%%%%%%%%%%%%%%%%%%%%%%%%%%%%%%%%%%%%%%%%%%%%%%%%%
\section{Variance Reduction in Learning Mean Flows}
\label{sec: method}

The MeanFlow loss in~\eqref{eq: mf-loss} is empirically known to be \textit{non-decreasing} and to suffer from \textit{high-variance gradients}~\cite{zhang2025alphaflowunderstandingimprovingmeanflow,geng2025improvedmeanflowschallenges}. We investigate the statistical mechanism behind this pathology. Importantly, our analysis isolates the role of the conditional velocity field $\vv_{\text{cond}}$ when it appears as the JVP tangent, and identifies that it acts as a Monte Carlo control variate for the inaccessible marginal velocity, with a coefficient that the original MeanFlow fixes to a suboptimal value. We derive the optimal coefficient and establish a connection to practical fixes in concurrent works.

%%%%%%%%%%%%%%%%%%%%%%%%%%%%%%%%%%%%%%%%%
\subsection{Two roles of the conditional velocity field}
\label{subsec: two-roles}

With Reynolds decomposition, we express the conditional velocity by $\vv_{\text{cond}}\!=\!\vv(\vx,t)+\vv^{\prime}$ with $\E_{\vx_{0}\mid\vx_{t}}[\vv^{\prime}]\!=\!\bm{0}$ according to~\eqref{eq: margv-condv-relation} and $\Sigma_{\vv^{\prime}}\!\triangleq\!\Cov_{\vx_{0}\mid\vx_{t}}[\vv^{\prime}]$. The MeanFlow loss in~\eqref{eq: mf-loss} uses $\vv_{\text{cond}}$ both as the \textit{regression target} and as the \textit{tangent} inside the total derivative. Carrying the decomposition through per-sample loss $\ell_{\text{MF}}(\vtheta)$ yields:
\begin{equation}
    \E_{\vx_{0}\mid\vx_{t}}\!\bigl[\ell_{\text{MF}}(\vtheta)\bigr]
    =\bigl\|\vr_{\vtheta}^{\texttt{sg}}\bigr\|^{2}_{2}
    +\texttt{sg}\!\bigl[\Tr(\mJ\Sigma_{\vv^{\prime}}\mJ^{\!\top})\bigr],
    \label{eq: per-sample-expand}
\end{equation}
where $\mJ\!\triangleq\!(t{-}r)\partial_{\vx_{t}}\vu_{\vtheta}\!-\!\mI_{d}\!\in\!\mathbb{R}^{d\times d}$ is a \textit{Jacobi factor} determined by the Jacobian matrix $\partial_{\vx_{t}}\vu_{\vtheta}$ and $\vr_{\vtheta}^{\texttt{sg}}\!\triangleq\!\vu_{\vtheta}\!+\!(t{-}r)\,\texttt{sg}[\partial_{\vx_{t}}\vu_{\vtheta}\!\cdot\!\vv\!+\!\partial_{t}\vu_{\vtheta}]\!-\!\vv$ is the deterministic mean-field residual. Since the residual $\vr_{\vtheta}^{\texttt{sg}}$ is the exact loss we aim to minimize given by~\eqref{eq: td-meanflow-identity}, the trace term leads to the non-decreasing loss and eliminates $\textit{iff}$ the Jacobi satisfies $\partial_{\vx_{t}}\vu_{\vtheta}=\frac{1}{t-r}\mI_{d}$ (equivalently, $\mJ\!=\!\bm{0}$). We identify two roles of the conditional velocity field in this decomposition:

\noindent\textbf{Unbiased Target.} As the regression target, $\vv_{\text{cond}}$ injects $\vv^{\prime}$ \textit{linearly} into the residual, contributing the \emph{irreducible} per-sample noise floor $\Tr(\Sigma_{\vv^{\prime}})$. However, since this noise is independent of $\vtheta$ and $\vv^{\prime}$ is zero-mean, $\vv_{\text{cond}}$ can serve as an unbiased target for learning Mean Flows.

\noindent\textbf{Variance Amplifier.} As the JVP tangent, the same $\vv_{\text{cond}}$ injects $\vv^{\prime}$ \textit{multiplicatively} through the Jacobi factor $\mJ$, contributing the \emph{amplified} term $\Tr(\mJ\Sigma_{\vv^{\prime}}\mJ^{\!\top})$. It scales unboundedly with the spectral magnitude of $\mJ$. This amplification is the dominant driver of variance at the gradient level:

\begin{theorem}[Jacobian Variance Amplification]
    Let $\vg\!\triangleq\!\nabla_{\vtheta}\vu_{\vtheta}(\vx_{t},r,t)\!\in\!\mathbb{R}^{d\times p}$ be the parameter Jacobian of the average velocity. The trace of the conditional gradient covariance (\textit{i.e., the total variance}) of $\ell_{\text{MF}}(\vtheta)$ in~\eqref{eq: mf-loss} satisfies
    \begin{equation}
        \boxed{
            \Tr\!\left(\Cov\!\left[\nabla_{\vtheta}\ell_{\text{MF}}\mid\vx_{t}\right]\right)
            \;\propto\;
            \Tr\!\left(\vg^{\!\top}\,\mJ\,\Sigma_{\vv^{\prime}}\,\mJ^{\!\top}\,\vg\right).
        }
    \end{equation}
    \label[theorem]{thm: jacobi-variance}
\end{theorem}

See the proof in~\Cref{appx: theorem-proof}. The theorem identifies that the total gradient variance grows with $\|\mJ\|^{2}$ for fixed $\vg$ and $\Sigma_{\vv^{\prime}}$, saturates only at the irreducible noise floor when $\mJ\!\to\!\bm{0}$, and is \emph{uncontrolled} by any term in the loss due to the use of the stop-gradient in the original MeanFlow. Meanwhile, without the stop-gradient operator, the optimizer could in principle reduce variance by driving $\mJ\!\to\!\bm{0}$. This establishes a \textit{semi-gradient gap}, formally,

\begin{theorem}[Semi-Gradient Gap]
    The gradient of the MeanFlow loss $\gL_{\text{MF}}$ with and without the stop-gradient operator differ by
    \begin{equation}
        \boxed{
            \underbrace{2(t{-}r)\,\E\!\left[\bigl(\nabla_{\vtheta}\!\left(\partial_{\vx_{t}}\vu_{\vtheta}\!\cdot\!\vv+\partial_{t}\vu_{\vtheta}\right)\bigr)^{\!\top}\vr_{\vtheta}\right]}_{\text{mean-field gradient difference}}
            +\underbrace{\E\!\left[\nabla_{\vtheta}\Tr\!\left(\mJ\Sigma_{\vv^{\prime}}\mJ^{\!\top}\right)\right]}_{\text{variance-driven gradient difference}}.
        }
    \end{equation}
    \label[theorem]{thm: semi-gradient-gap}
\end{theorem}

In the original MeanFlow, the stop-gradient operator prevents $\mJ$ from being passed to the optimizer, resulting in an empirically non-decreasing loss. Meanwhile, the mean-field difference vanishes at convergence ($\vr_{\vtheta}\!\to\!\bm{0}$) and the variance-driven term $\Tr(\mJ\Sigma_{\vv^{\prime}}\mJ^{\!\top})$ dominates. To better illustrate the idea, \Cref{fig: cm-gap} visualizes the spatial distribution of $\sqrt{\E_{\vx_{0}\mid{\vx_{t}}}\left\|\vv^{\prime}\right\|^{2}}$ on a three-mode, two-dimensional Gaussian mixture. It shows that the total variance magnitude is non-zero, concentrates in mode-mixing regions where conditional paths overlap, and grows toward the latent endpoint ($t\!\to\!1$).

\begin{figure}[t]
    \centering
    \includegraphics[width=\textwidth]{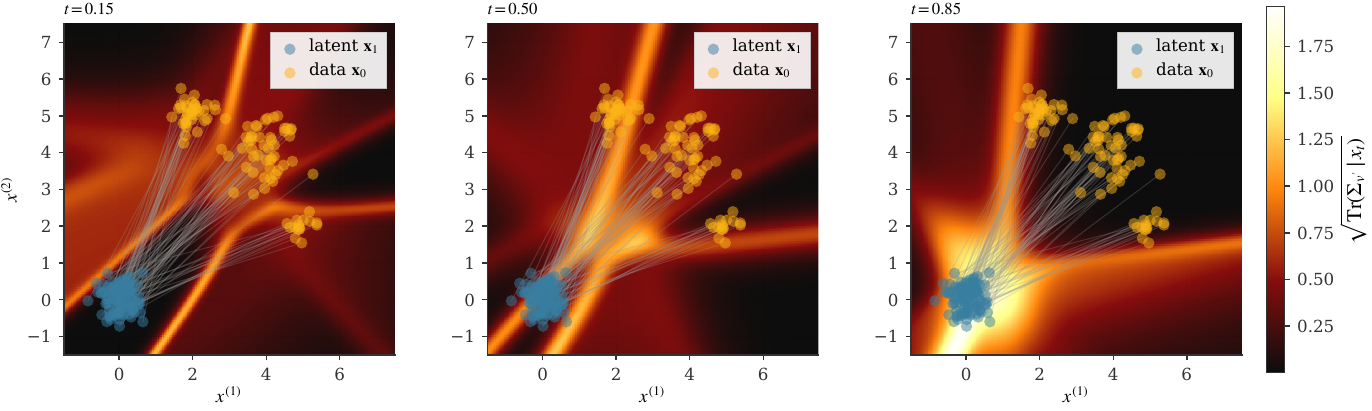}
    \caption{Spatial Distribution of $\sqrt{\Tr(\Sigma_{\vv^{\prime}}\!\mid\!\vx_{t})}\!=\!\sqrt{\E_{\vx_{0}\mid\vx_{t}}\|\vv^{\prime}\|^{2}}$ at Three Timesteps on a 2-D Gaussian Mixture. Conditional variances concentrate in mode-mixing regions.}
    \label{fig: cm-gap}
    \vspace{-1em}
\end{figure}

%%%%%%%%%%%%%%%%%%%%%%%%%%%%%%%%%%%%%%%%%
\subsection{Rethinking tangent as a control variate}
\label{subsec: control-variate}

The amplification in~\Cref{thm: jacobi-variance} arises due to using $\vv_{\text{cond}}\!=\!\vv\!+\!\vv^{\prime}$ as the JVP tangent injects the zero-mean fluctuation $\vv^{\prime}$. Hence, $\vv_{\text{cond}}$ is analogous to the canonical structure of a Monte Carlo \textit{control variate}~\cite{glasserman2003montecarlo}, where $\vv_{\text{cond}}$ is a stochastic estimator of the inaccessible $\vv$ with known mean, and the coefficient on the noise term $\vv^{\prime}$ determines the bias-variance trade-off. To make this explicit, we introduce a tangent-mixing coefficient $\beta\!\in\![0,1]$:
\begin{equation}
    \tilde{\vv}_{\text{tang}}(\beta)\;\triangleq\;(1-\beta)\,\vv_{\text{cond}} + \beta\,\hat{\vv}\;=\;\vv + (1-\beta)\,\vv^{\prime} + \beta\,\vb,
    \label{eq: tangent-mix}
\end{equation}
where $\hat{\vv}$ can be any deterministic proxy for the marginal velocity (\textit{e.g., a velocity network}) with bias $\vb\!\triangleq\!\hat{\vv}-\vv$ that is deterministic given $\vx_{t}$. Herein, $\beta\!=\!0$ recovers vanilla MeanFlow and $\beta\!=\!1$ replaces the tangent entirely with the deterministic proxy.

\noindent\textbf{Bias-variance trade-off.} Substituting $\tilde{\vv}_{\text{tang}}(\beta)$ for $\vv_{\text{cond}}$ in the JVP tangent and propagating through the loss, the per-sample gradient writes
\begin{equation}
    \vg^{(\beta)} \;=\; 2\,\vg^{\!\top}\!\Bigl[\vr_{\vtheta} \;+\; \bigl((1{-}\beta)\,\mJ - \beta\,\mI\bigr)\,\vv^{\prime} \;+\; \beta\,(\mJ{+}\mI)\,\vb\Bigr],
    \label{eq: g-beta}
\end{equation}
where the noise term scales the conditional fluctuation $\vv^{\prime}$ by the matrix $((1{-}\beta)\,\mJ - \beta\,\mI)$, and the bias term scales the proxy error $\vb$ by $\beta\,(\mJ{+}\mI)$. In this case, we can balance the bias-variance trade-off by solving for the optimal coefficient $\beta^{\ast}$ that minimizes the least-squares between the expected per-sample gradient in~\eqref{eq: g-beta} and the target gradients, given by
\begin{equation}
    M(\beta)\;\triangleq\;\E_{\vv^{\prime}\mid{\vx_{t}}}\left[\left\|\vg^{(\beta)}-\nabla_{\vtheta}\|\vr_{\vtheta}\|_{2}^{2}\right\|_{2}^{2}\right]\;=\;\E_{\vv^{\prime}\mid{\vx_{t}}}\left[\left\|\vg^{(\beta)}-2\vg^{\top}\vr_{\vtheta}\right\|_{2}^{2}\right].
    \label{eq: beta-target}
\end{equation}
We derive the following theorem.

\begin{theorem}[Optimal control-variate coefficient]
    Under the scalar-isotropic approximation $\Sigma_{\vv^{\prime}}\!\approx\!\sigma^{2}\mI_{d}$, $\mJ\!\approx\!\kappa\mI_{d}$, and the parameter-isotropy approximation $\vg\vg^{\!\top}\!\propto\!\mI_{d}$ (which reduces $M$ to a per-component MSE; cf.~\Cref{appx: theorem-proof}),
    \begin{equation}
        M(\beta) \;\propto\; \beta^{2}(\kappa{+}1)^{2}\|\vb\|^{2} + \sigma^{2}d\bigl((1{-}\beta)\kappa - \beta\bigr)^{2}.
    \end{equation}
    For $\kappa\!>\!0$, $M(\beta)$ admits a unique minimizer
    \begin{equation}
        \boxed{
            \beta^{\ast} \;=\; \underbrace{\frac{\kappa}{\kappa+1}}_{\text{noise-cancellation}}\;\cdot\;\underbrace{\frac{\sigma^{2}d}{\sigma^{2}d + \|\vb\|^{2}}}_{\text{shrinkage}},
        }
        \label{eq: beta-minimizer}
    \end{equation}
    with optimum value $M(\beta^{\ast})\!\propto\!\sigma^{2}d\,\kappa^{2}\,\|\vb\|^{2}/(\sigma^{2}d+\|\vb\|^{2})$.
    % For all $\kappa,\|\vb\|^{2}\!>\!0$, $M(\beta^{\ast})$ strictly dominates both corner MSEs $M(0)\!\propto\!\sigma^{2}\kappa^{2}d$ (vanilla MeanFlow) and $M(1)\!\propto\!(\kappa{+}1)^{2}\|\vb\|^{2}+\sigma^{2}d$ (deterministic-tangent estimator with bias). The boundary case $\kappa\!=\!-1$ (which corresponds to $\partial_{\vx_{t}}\vu_{\vtheta}\!=\!\bm{0}$ at initialization) leaves $M(\beta)$ constant and is excluded.
    \label[theorem]{thm: control-variate-optimum}
\end{theorem}

See the proof in~\Cref{appx: theorem-proof}. The two components of $\beta^{\ast}$ play distinct roles: $\kappa/(\kappa{+}1)$ is the coefficient that \emph{cancels the Jacobian-amplified noise} through destructive interference between the $\mJ\vv^{\prime}$ and $\mI\vv^{\prime}$ contributions in~\eqref{eq: g-beta}, while $\sigma^{2}d/(\sigma^{2}d{+}\|\vb\|^{2})$ is the standard James-Stein shrinkage~\cite{james1961estimation} that pulls $\beta^{\ast}$ toward the unbiased corner when the proxy bias $\|\vb\|$ is large. Therefore, as the proxy bias shrinks ($\|\vb\|^{2}\!\to\!0$), $\beta^{\ast}\!\to\!\kappa/(\kappa{+}1)$, and as the Jacobian magnitude grows ($\kappa\!\to\!\infty$), $\beta^{\ast}\!\to\!1$. Following the theory, we explain the high variance of gradients in the original MeanFlow and establish a connection to practical fixes in concurrent work.

\noindent\textbf{High-variance gradient with suboptimal coefficient.} The original MeanFlow uses $\beta\!=\!0$ which is optimal only when $\kappa\!=\!0$ (i.e., $\partial_{\vx_{t}}\vu_{\vtheta}\!=\!\frac{1}{t-r}\mI_{d}$) or when no deterministic proxy is available ($\|\vb\|\!=\!\infty$). Both conditions fail in practice: trained backbones produce $\partial_{\vx_{t}}\vu_{\vtheta}$ matrices that are spectrally far from $\frac{1}{t-r}\mI_{d}$, and there are multiple viable deterministic proxies for the marginal $\vv$.

\input{assets/tables/concurrent}

\noindent\textbf{Concurrent works as control variates.} Our theory unifies several remedies for MeanFlow's training pathology across concurrent works into a single one-parameter family. Each empirical fix corresponds to a different practical realization of the optimum $\beta^{\ast}$, distinguished by \emph{which proxy} they use for $\vv$ and \emph{which value of $\beta$} the construction implicitly selects. We summarize them in~\cref{tab: concurrent-positioning}.

%% file: assets/tables/concurrent.tex
\begin{table}[!t]
    \centering
    \caption{Concurrent MeanFlow remedies as special cases in our control-variate framework. Each method either drives $\beta\!\to\!1$ with a deterministic proxy, directly reduces the norm of amplification factor $\|\mJ\|^{2}$, or replaces the JVP construction altogether.}
    \label{tab: concurrent-positioning}
    \begin{tabular}{@{}lll@{}}
        \toprule
        \textbf{Method} & \textbf{Mechanism in our framework} & \textbf{Practical realization} \\
        \midrule
        AlphaFlow~\cite{zhang2025alphaflowunderstandingimprovingmeanflow} & avoid high-$\kappa$ regimes & $\alpha$-curriculum schedule \\
        Improved MF~\cite{geng2025improvedmeanflowschallenges} & deterministic tangent ($\beta\!\to\!1$) & separate velocity head + sg \\
        Modular MF~\cite{you2025modularmeanflow} & deterministic tangent ($\beta\!\to\!1$) & gradient-modulated loss family \\
        Kim~\textit{et al.}~\cite{kim2025understandingacceleratingmeanflow} & deterministic tangent ($\beta\!\to\!1$) & staged training schedule \\
        TVM~\cite{zhou2026terminalvelocitymatching} & remove the spatial Jacobian ($\mJ\!\to\!\bm{0}$) & terminal-time differentiation \\
        Re-MeanFlow~\cite{zhang2026overcomingcurvaturebottleneckmeanflow} & reduce $\|\mJ\|$ & rectified-flow preprocessing \\
        Decoupled MF~\cite{lee2025decoupledmeanflow} & bypass JVP construction & flow-map conditioning on later $t$ \\
        \bottomrule
    \end{tabular}
    \vspace{-1em}
\end{table}

%% file: contents/related.tex
\section{Related Work}
\label{sec: review}

\noindent\textbf{Variance reduction in learning flow-based models.} Several recent works address variance in learning flow-based models from complementary perspectives. Stable Velocity~\cite{yang2026stablevelocityvarianceperspective} reveals a two-regime variance structure in flow matching and proposes composite conditioning; TPC~\cite{maduabuchi2026temporalpairconsistencyvariancereduced} reduces variance through temporal pair coupling; and preconditioned flow matching~\cite{ahamed2026preconditionedscoreflowmatching} addresses ill-conditioning via anisotropic preconditioning. Bertrand et al.~\cite{bertrand2025closedformflowmatchinggeneralization} argue that stochastic-target variance is negligible for standard flow matching, consistent with our analysis, since the problematic amplification we identify is specific to the MeanFlow JVP, not the flow-matching loss itself. Rectified flows~\cite{liu2022flowstraightfastlearning,lee2024improvingrectifiedflows} reduce trajectory curvature by iterative straightening, which by~\Cref{thm: jacobi-variance} reduces $\|\mJ\|^{2}$ and hence the amplified noise term.

\noindent\textbf{Control variates and target networks.} The control-variate framework we apply has classical roots in Monte Carlo estimation~\cite{glynn2002some}; it is also the analytical structure underlying target networks in deep RL~\cite{mnih2015human} and self-distillation in self-supervised learning~\cite{grill2020bootstrap}. We make the connection explicit by interpreting the EMA-derived tangent in MeanFlow as a target-network-style proxy that drives the control-variate coefficient toward its optimum.

%% file: contents/experiment.tex
\section{Experiments}
\label{sec: experiment}

We empirically validate our theory in~\Cref{sec: method} by probing the bias-variance trade-offs. Specifically, we sweep the tangent-mixing coefficient $\beta\!\in\![0,1]$ end-to-end and compare the results to the closed-form prediction based on~\Cref{thm: control-variate-optimum}.

%%%%%%%%%%%%%%%%%%%%%%%%%%%%%%%%%%%%%%%%%
\noindent\textbf{Instantiation of the $\beta\!=\!1$ corner.}
To validate the framework empirically, we need a concrete realization of $\beta$ that we can sweep at scale. We replace the JVP tangent by an exponential moving-average~\cite{polyak1992acceleration} deterministic proxy $\hat{\vv}\!\gets\!\vu_{\bar{\vtheta}}(\vx_{t},t,t)$, and use a small flow-matching loss at $r\!=\!t$ to keep the proxy bias $\vb\!=\!\hat{\vv}-\vv$ small and bound the bias term $\beta(\mJ{+}\mI)\vb$ in~\eqref{eq: g-beta}. We keep the regression target $\vv_{\text{cond}}$ unchanged, exploiting the role asymmetry identified in~\cref{subsec: two-roles}. \Cref{appx: implementation} provides details about the full loss, training algorithm, and three propositions characterizing the resulting gradient form, residual bias, and the necessity of the target-tangent role split.

\noindent\textbf{Experiment Setup.} We evaluate our theory across three regimes: six two-dimensional toy datasets, dense Gaussian mixtures (DGMM) with various dimensions $d\!\in\!\{2,4,8,16,32,64\}$, and DiT-B/4~\cite{peebles2023scalablediffusion} trained on ImageNet with Stable Diffusion latents\footnote{\url{https://huggingface.co/enterprise-explorers/sd-vae-ft-mse-flax}}. On the 2-D toy and DGMM datasets, we use a three-layer MLP with $128$ hidden units, train for $200{,}000$ steps with batch size $256$ and three random seeds $\{42, 0, 1\}$, and sweep the eleven-point grid $\beta\!\in\!\{0.0,0.1,\dots,1.0\}$ for the sample-quality, evaluated by final-step sliced Wasserstein-$1$ distance ($\mathrm{SW}_1$) on $4096$ samples with $500$ random projections under a fixed projection seed. On the toy datasets, we additionally run a sub-sweep at $\beta\!\in\!\{0,0.25,0.5,0.75,1\}$ that investigates the per-step total gradient variance $\Tr(\Cov[\nabla_{\vtheta}\ell_{\text{MF}}])$ using $K\!=\!8$ replica of mini-batches, each with $256$ samples, every $2{,}000$ steps and tail-averages over the second half of training. For the DiT, we perform a four-point sweep $\beta\!\in\!\{0,0.25,0.5,1\}$ and train the model for $300$k steps following the recipe in the original MeanFlow~\cite{geng2025meanflowsonestepgenerative}. We defer the full setup, FID table, and direct matrix-form measurement of $\beta^{\ast}$ to~\Cref{appx: dit-results}.

\begin{figure}[!t]
    \centering
    \includegraphics[width=\textwidth]{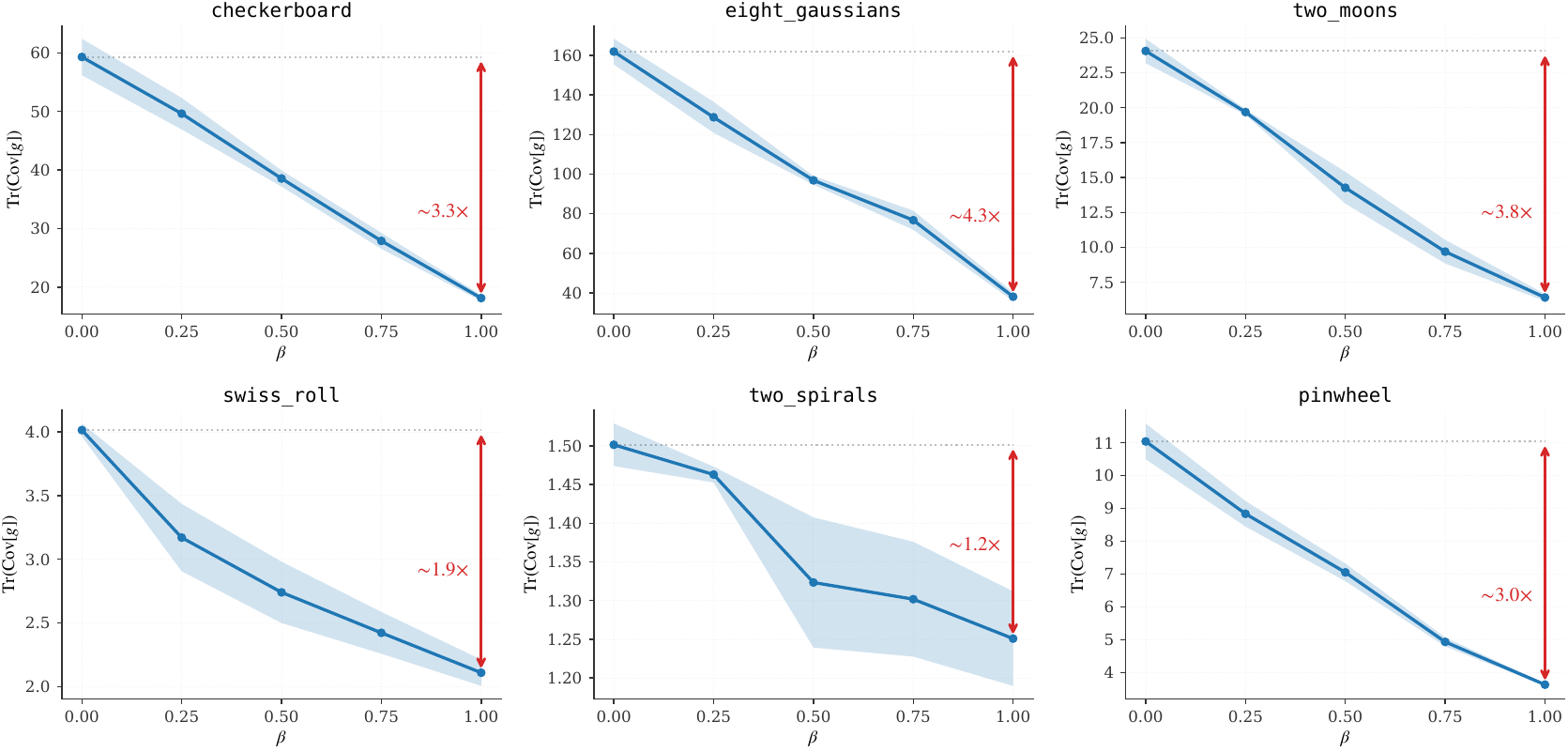}
    \caption{Total Gradient Variance $\Tr(\Cov[\nabla_{\vtheta}\ell_{\text{MF}}])$ on toy datasets with $\beta\!\in\!\{0,0.25,0.5,0.75,1\}$. Variance decreases monotonically in $\beta$ on almost every dataset.}
    \label{fig: nr-vs-beta}
    \vspace{-1em}
\end{figure}

\subsection{Variance reduction}
\label{subsec: exp-grad-var}

\Cref{fig: nr-vs-beta} reports the total gradient variances $\Tr(\Cov[\nabla_{\vtheta}\ell_{\text{MF}}])$ on all six 2-D datasets. We observe monotonically decreasing variances with respect to $\beta$ on almost every dataset. The empirical ordering tracks the variance term $\sigma^{2}d\bigl((1{-}\beta)\kappa - \beta\bigr)^{2}$ from~\Cref{thm: jacobi-variance} and~\Cref{thm: control-variate-optimum}. The result confirms that the conditional fluctuation $\vv^{\prime}$ is the dominant variance source in the MeanFlow gradients. Meanwhile, we observe that the magnitude of the reduction varies with the spectral magnitude $\|\mJ\|$ of the backbone model. The largest reductions appear on the \emph{high-curvature} datasets (\emph{e.g.,} \texttt{eight\_gaussians}, \texttt{checkerboard}), and the smallest ($1.20\!\times$) on \texttt{two\_spirals}. The observation is consistent with our theory: low geometric curvature pushes the network toward the flat regime $\partial_{\vx_{t}}\vu_{\vtheta}\!\to\!\frac{1}{t-r}\mI_{d}$, corresponding to $\kappa\!\approx\!0$, and our~\cref{thm: control-variate-optimum} predicts both a small per-step variance reduction and a left-shift of optimal coefficient $\beta^{\ast}\!\to\!0$ given by the noise-cancellation factor $\kappa/(\kappa{+}1)$.

\begin{figure}[!t]
    \centering
    \includegraphics[width=\textwidth]{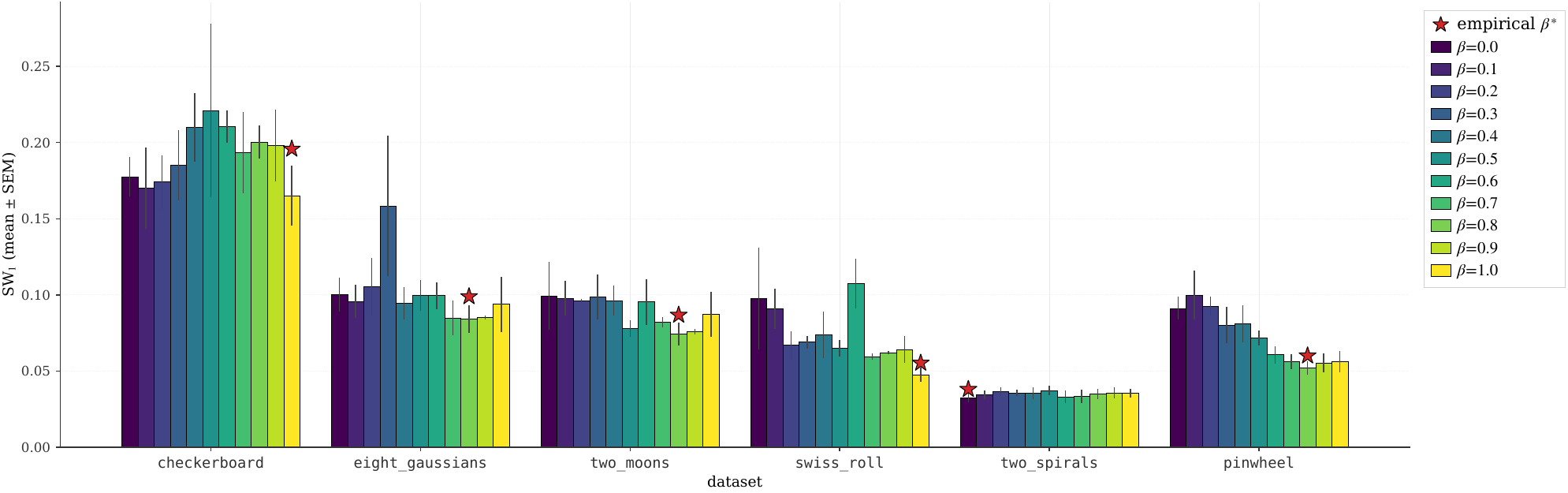}
    \caption{Sample Quality (Sliced Wasserstein-$1$ Distance $\mathrm{SW}_{1}(\beta)$) on toy datasets with $\beta\!\in\!\{0.0,0.1,\dots,1.0\}$. Bars are grouped by dataset and colored by $\beta$ (standard error bars). The red star marks the empirical optimum $\beta^{\ast}$.}
    \label{fig: sw1-vs-beta}
\end{figure}

\subsection{Bias-variance trade-offs}
\label{subsec: exp-m-beta}

Whereas $\Tr(\Cov[\nabla_{\vtheta}\ell_{\text{MF}}])$ reveals the variance reduction induced by $\beta$, the sample-quality measured by $\mathrm{SW}_{1}(\beta)$ exposes the full bias-variance trade-off. \Cref{fig: sw1-vs-beta} traces $\mathrm{SW}_{1}(\beta)$ across the full $11$-point grid on six 2-D toy datasets. On the five high-curvature datasets, the empirical $\beta^{\ast}$ falls in between $\{0.8,1.0\}$, recovering the predicted $\beta\!\to\!1$ location with a small bias norm $\|\vb\|_{2}^{2}$ when the EMA proxy is accurate. On the lower-curvature \texttt{two\_spirals} dataset, consistent with~\cref{subsec: exp-grad-var}, the geometry drives $\partial_{\vx_{t}}\vu_{\vtheta}\!\to\!\frac{1}{t-r}\mI_{d}$, essentially pushing $\kappa$ toward zero. The predicted $\beta^{\ast}$ then shrinks toward the unbiased corner $\beta\!=\!0$ through the noise-cancellation factor $\kappa/(\kappa{+}1)$ in~\Cref{thm: control-variate-optimum}, and the empirical sweep result concurs. Meanwhile, we are surprised that the $\beta^{\ast}=1$ instantiation reduces SW$_{1}$ by $54\%$ compared to the original MeanFlow ($\beta=0$) on the \texttt{swiss\_roll} dataset.

\input{assets/tables/dgmm-beta-summary}

\subsection{Scaling with feature dimensions}
\label{subsec: exp-dgmm}

We investigate whether our theory holds as the feature dimension grows by sweeping the coefficient $\beta\!\in\!\{0,0.1,\dots,1\}$ on dense Gaussian mixture (DGMM) datasets in $\mathbb{R}^{d}$, varying $d\!\in\!\{2,4,8,16,32,64\}$. \Cref{tab: dgmm-beta-summary} summarizes the sample-quality gap between the original MeanFlow with $\beta\!=\!0$ and the empirical $\beta^{\ast}$ at each $d$. We report full results in~\Cref{tab: dgmm-beta-full} of~\Cref{appx: results}. On low-dimensional datasets with $d\!\in\!\{2,4,8,16,32\}$, SW$_{1}$ is essentially \emph{$\beta$-insensitive}, where the quality gap between $\beta\!=\!0$ and $\beta^{\ast}$ stays within one standard error of the mean (SEM). This is consistent with the prediction from~\cref{thm: control-variate-optimum} when the EMA proxy yields small $\|\vb\|_{2}^{2}$. The noise-cancellation factor $\kappa/(\kappa{+}1)$ saturates near $1$ across the whole $\beta$ range, leading to flat SW$_{1}(\beta)$. At $d\!=\!64$ the data hints at an interior minimum at $\beta^{\ast}\!=\!0.4$ with a $\sim\!9\%$ reduction over $\beta\!=\!0$, consistent with $\kappa$ coming off saturation with increasing data dimension $d$. In~\cref{subsec: exp-dit}, we show that the same observation holds for a larger and more complex DiT model, where $\beta^{\ast}_{\mathrm{matrix}}\!\approx\!0.94$ also sits short of the deterministic-tangent corner.

\subsection{Latent Diffusion Transformers on ImageNet}
\label{subsec: exp-dit}

We further investigate whether our theory generalizes to large models with high-dimensional features. \Cref{fig: dit-fid-curves} shows two complementary measurements from the experiments with DiT-B/4 on ImageNet. The left panel in~\Cref{fig: dit-fid-curves} validates that the gradient-variance reduction predicted by~\Cref{thm: jacobi-variance} carries over from 2-D toy benchmark datasets to DiT. Herein, we directly measure the per-step loss variance using the model trained with the original MeanFlow loss, using two distinct tangents: the stochastic conditional-velocity tangent $\vv_{\text{cond}}$ and the deterministic EMA tangent $\vu_{\bar{\vtheta}}(\vx_{t},t,t)$. We observe that the loss variance with deterministic-tangent stays roughly flat across $t$, while the stochastic-tangent variance grows by two orders of magnitude. The widening gap directly validates the Jacobi-factor amplification of~\Cref{thm: jacobi-variance}, where the stochastic tangent passes the conditional-velocity noise through $\mJ\!=\!(t{-}r)\partial_{\vx_{t}}\vu{-}\mI$, which the deterministic tangent eliminates by construction. We observe the same pattern using the checkpoint of the $\beta\!=\!1$ instantiation (see~\cref{appx: dit-results}), confirming that the amplification is a property of the architecture and data.

The right panel in~\Cref{fig: dit-fid-curves} visualizes the converged FID as a function of $\beta$, with $y$-axis the empirical $\Delta\mathrm{FID}(\beta)\!\triangleq\!\mathrm{FID}(\beta)\!-\!\mathrm{FID}(\beta=0)$. The result shows that the coarse ordering $\mathrm{FID}(\beta\!=\!1)\!\gg\!\mathrm{FID}(\beta\!\le\!0.5)$ and $\mathrm{FID}(\beta\!=\!0.5)\!>\!\mathrm{FID}(\beta\!=\!0)$ holds at every matched-step checkpoint, and the full four-point ordering holds at convergence; we discuss the seed-sensitivity of the finer $\beta\!=\!0$ vs $\beta\!=\!0.25$ gap in~\Cref{appx: dit-results}. We report per-checkpoint values in~\Cref{tab: dit-fid} of~\Cref{appx: dit-results}. The result establishes the following observation:

\begin{figure}[t]
    \centering
    \includegraphics[width=\textwidth]{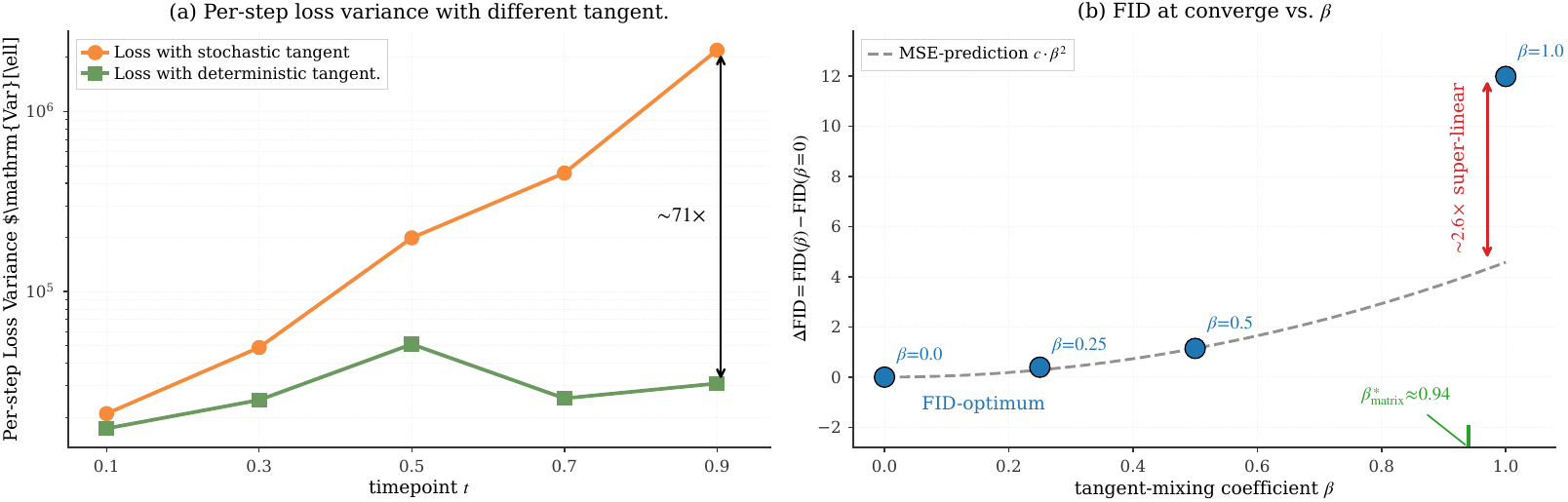}
    \caption{DiT-B/4 on ImageNet-$256$. \emph{Left:} Per-step loss variance under the stochastic ($\beta\!=\!0$) versus deterministic ($\beta\!=\!1$) tangent, measured on the baseline ($\beta=0$) checkpoint with matched data and noise. \emph{Right:} Converged FID versus $\beta$, reported as $\Delta\mathrm{FID}(\beta)\!=\!\mathrm{FID}(\beta)\!-\!\mathrm{FID}(\beta=0)$ at matched step $295$k; the dashed gray curve is the quadratic $c\beta^{2}$.}
    \label{fig: dit-fid-curves}
    \vspace{-1em}
\end{figure}

\noindent\textbf{FID-MSE landscape mismatch.} We report the direct measurement of both ingredients in~\Cref{thm: control-variate-optimum} on the DiT checkpoints, including the EMA-tracking proxy $\widehat{\|\vb\|^{2}}$ and the irreducible noise floor $\sigma^{2}d$, in~\Cref{appx: dit-results}. The result predicts an optimal $\beta^{\ast}$ in the interior range $[0.46,\,0.50]$ under the scalar-isotropic approximation, while the matrix-form $\beta^{\ast}_{\text{matrix}}$ is even larger (see below). The empirical ordering of FID at convergence is hence consistent with the prediction. To test whether FID tracks the gradient-MSE \emph{quantitatively}, we fit a MSE-prediction curve $c\beta^{2}$ and compare the $\beta\!=\!1$ offset, where $c$ is given by
\begin{equation}
    c\gets\underset{c\in\R}{\arg\min}\left\|\Delta\mathrm{FID}(\beta)-c\beta^{2}\right\|_{2}^{2},\qquad\beta\in\left\{0.0, 0.25, 0.5\right\}.
\end{equation}
\Cref{fig: dit-fid-curves} shows a $\sim\!2.6\!\times$ super-linear amplification along the $\Delta\mathrm{FID}$ axis, indicating that the FID over-penalizes bias relative to gradient MSE. As a consequence, the FID-optimal $\beta$ shifts past the gradient-MSE interior minimizer all the way to the unbiased corner $\beta\!=\!0$.

To verify that our isotropic approximation in~\cref{thm: control-variate-optimum} is not artifactual, we directly measured a matrix-form upper bound on $\beta^{\ast}_{\text{matrix}}$ on the baseline checkpoint by sampling actual $\vv^{\prime}$ values from the conditional distribution (see~\Cref{appx: theorem-proof}). The estimator omits the bias term $\|(\mJ{+}\mI)\vb\|^{2}$ from the denominator, giving
\begin{equation}
    \beta^{\ast}_{\text{no bias}}\!\triangleq\!\Tr(\mJ\Sigma_{\vv^{\prime}}(\mJ{+}\mI)^{\!\top})/\Tr((\mJ{+}\mI)\Sigma_{\vv^{\prime}}(\mJ{+}\mI)^{\!\top})\;\geq\;|\beta^{\ast}_{\text{matrix}}|,
\end{equation}
since the omitted bias term is \emph{non-negative}. We evaluate the bound at $t\!\in\!\{0.1,0.3,0.5,0.7,0.9\}$ with fixed gap $t-r\!=\!0.25$, and $512$ samples per $t$ using $\vu_{\vtheta}(\vx_{t},t,t)$ as the marginal-velocity proxy. The numerator is \emph{strictly positive} at every probed $t$, ruling out the unbiased-corner regime $\beta^{\ast}_{\text{matrix}}\!\le\!0$. Aggregated across $t$, we have $\beta^{\ast}_{\text{no bias}}\!\approx\!0.94$, which is close to the deterministic-tangent corner and \emph{larger} than the scalar-isotropic estimate $\beta^{\ast}\!\in\![0.46,0.50]$. The full $\beta^{\ast}_{\text{matrix}}$ then lies in $(0,0.94]$. Herein, we do not pin down the exact value due to the inaccessible true model bias $\vb\!=\!\vu_{\vtheta}(\vx_{t},t,t)-\vv(\vx_{t},t)$ on ImageNet without access to the conditional distribution $p(\vx_{0}\!\mid\!\vx_{t})$. The matrix-form measurement further strengthens, rather than weakens, the FID-MSE mismatch: the gradient-MSE optimum sits in $(0,0.94]$ while the FID optimum sits at the unbiased corner $\beta\!=\!0$.

\noindent\textbf{Independent validation of the small-bias regime.} The estimate above uses the boundary $\vu_{\vtheta}(\vx_{t},t,t)$ of the model as a proxy. Hence, the measured bias is referenced to the same network. To break this circularity, we train a separate, unconditional flow-matching model on the \emph{same latents and backbone} $\vv_{\text{ref}}$. It is therefore a reference model in the low-bias regime, with a boundary residual matching the analytic noise floor to within $\sim\!1.7\%$. We then re-estimating a non-cirular bias by $\|\vu_{\vtheta}(\vx_{t},t,t)-\vv_{\text{ref}}\|^{2}$, which gives a ratio $\|\vb\|^{2}/\sigma^{2}d\!\in\![0.006,0.043]$ across flow time $t$. Meanwhile, the matrix-form $\beta^{\ast}$ recomputed with this independent bias stays in $[0.901,0.935]$ (see \Cref{appx: vref}). The model bias is therefore genuinely small and $\beta^{\ast}_{\text{matrix}}$ sits near the top of $(0,0.94]$ on an independent footing, sharpening the FID-MSE mismatch.

%% file: assets/tables/dgmm-beta-summary.tex
\begin{table}[!t]
    \centering
    \caption{DGMM $\beta$-sweep summary across dimensions $d\!\in\!\{2,4,8,16,32,64\}$. The quality gap column marks reduction that exceed one standard error mean with $\dagger$. $\mathrm{SW}_{1}$ is essentially $\beta$-insensitive at $d\!\le\!32$, while only $d\!=\!64$ shows a tentative interior minimum.}
    \label{tab: dgmm-beta-summary}
    \begin{tabular}{@{}c|cccc@{}}
        \toprule
        $d$ & $\mathrm{SW}_{1}(\beta\!=\!0)$ & $\mathrm{SW}_{1}(\beta=\beta^{\ast})$ & \textbf{Quality Gap} & $\beta^{\ast}$ \\
        \midrule
        $2$  & $0.152{\scriptstyle\pm0.022}$ & $0.146{\scriptstyle\pm0.015}$ & $\sim\!4\%$  & $1.0$ \\
        $4$  & $0.083{\scriptstyle\pm0.011}$ & $0.083{\scriptstyle\pm0.011}$ & $<\!1\%$    & $0.8$ \\
        $8$  & $0.062{\scriptstyle\pm0.003}$ & $0.062{\scriptstyle\pm0.003}$ & $\sim\!1\%$ & $1.0$ \\
        $16$ & $0.044{\scriptstyle\pm0.002}$ & $0.044{\scriptstyle\pm0.003}$ & $<\!1\%$    & $1.0$ \\
        $32$ & $0.034{\scriptstyle\pm0.001}$ & $0.033{\scriptstyle\pm0.001}$ & $<\!1\%$    & $0.1$ \\
        $64$ & $0.028{\scriptstyle\pm0.003}$ & $0.025{\scriptstyle\pm0.001}$ & $\sim\!9\%^{\dagger}$ & $0.4$ \\
        \bottomrule
    \end{tabular}
\end{table}

%% file: contents/conclusion.tex
\section{Conclusion}
\label{sec: conclusion}
In this paper, we establish a theory, from the perspective of Monte Carlo control variates, that isolates the statistical structure of training pathology in the original MeanFlow. We identify two distinct roles the conditional velocity plays in the loss, and the original objective assigns the wrong coefficient to one of them. Following the theory, we derive the optimal tangent-mixing coefficient for the control variate in closed form. In our experiment, we conduct a controlled sweep of the coefficient to recover the predicted bias-variance trade-offs, scaling from two-dimensional toy datasets to a DiT-B/4 model on ImageNet-$256$. In addition, our experiment with DiT reveals a quantitative \emph{FID-MSE landscape mismatch}, where the optimal coefficient given by minimizing matrix-form MSE sits at $\beta^{\ast}\!\approx\!0.94$. In contrast, the optimal coefficient minimizing the FID leans towards the unbiased corner $\beta\!=\!0$, with a super-linear relationship between the empirical $\mathrm{FID}(\beta\!=\!1)$ and the bias. We discuss practical insights and limitations below.

\noindent\textbf{Practical insights.} Selecting $\beta$ trades off the data-dependent $\kappa$ against the model bias $\|\vb\|^{2}$ of~\Cref{thm: control-variate-optimum}. The deterministic-tangent corner ($\beta\!\to\!1$), realized by an EMA proxy at $O(1)$ cost, minimizes the \emph{per-step} gradient variance, but this does not by itself confer optimization stability: because the EMA tangent lags the online parameters, a large learning rate makes it stale and injects a compounding bias, and in a learning-rate sweep on \texttt{eight\_gaussians} and \texttt{checkerboard} (\Cref{appx: stability-pilot}) the $\beta\!=\!1$ instantiation in fact diverged at a \emph{lower} learning rate than vanilla MeanFlow, consistent with our framework, since the variance reduction is a per-step property while the EMA-induced bias enters multiplicatively through $\beta(\mJ{+}\mI)\vb$. When sample quality is the goal, our DiT experiment instead favors $\beta$ near $0$, or annealing $\beta$ with $\|\vb\|_{2}^{2}$ until the FID bias penalty vanishes. A metric-aware schedule that interpolates between these regimes and monitors proxy drift, rather than assuming the deterministic tangent is uniformly more stable, is a promising direction.

\noindent\textbf{Limitations.} We acknowledge several limitations of this work: \emph{(i) Scalar-isotropic approximation.} Our~\Cref{thm: control-variate-optimum} is derived under $\mJ\!\approx\!\kappa\mI_{d}$ and $\Sigma_{\vv^{\prime}}\!\approx\!\sigma^{2}\mI_{d}$. In the full diagonal matrix-valued setting the optimum is direction-dependent: a per-direction $\beta_{i}\!=\!\frac{\mJ_{ii}}{\mJ_{ii}+1}\!\cdot\!\frac{\sigma_{i}^{2}}{\sigma_{i}^{2}+b_{i}^{2}}$ would minimize the per-component gradient MSE, and a single global $\beta$ trades off residual variance across directions. This assumption may be valid for backbones with a concentrated Jacobian spectrum, but it may fail for more complex models or highly anisotropic data. \emph{(ii) FID-aware loss design.} Our experiment characterizes the FID-MSE mismatch but does not derive an FID-aware loss or $\beta$ schedule, which remains an open question. \emph{(iii) Non-Euclidean data.} We conduct our experiment in the Euclidean space $\mathbb{R}^{d}$. Extending the closed-form $\beta^{\ast}$ to manifold MeanFlow~\cite{woo2026riemannianmeanflow} is straightforward at the gradient level but requires care in defining the second moment $\Tr(\Sigma_{\vv^{\prime}})$ on tangent spaces.

\noindent\textbf{Future work.} We see three directions worth pursuing. First, we argue that our control-variate diagnosis applies to any self-supervised one-step generative model that bootstraps a parametrized map as the JVP-style total derivative in MeanFlow, including consistency models~\cite{song2023consistencymodels} and shortcut models~\cite{frans2025stepdiffusionshortcutmodels}. Second, a per-sample scalar coefficient, driven by online estimates of $\kappa$ and $\|\vb\|_{2}^{2}$, can potentially close the global-vs-per-sample and scalar-vs-matrix gaps during training. Finally, an adaptive schedule that interpolates between the gradient-MSE optimum and the FID-optimal corner, calibrated by either a learned proxy or proxy-quality validation FID, could resolve the FID-MSE mismatch we identified.

%% file: contents/postscripts.tex
\begin{ack}
    The authors appreciate the Google TPU Research Cloud (TRC) for supporting access to TPUs and Dr. Jiaru Zhang for his insightful discussion.
\end{ack}

%% file: contents/supplementary.tex
\newpage
\appendix
\setcounter{theorem}{0}
\begin{appendices}
    % 1. Start a new local contents list
    \startcontents[appendices]

    % 2. Print the local table of contents (only for what follows)
    \printcontents[appendices]{l}{1}{\setcounter{tocdepth}{2}}
    \newpage

    %%%%%%%%%%%%%%%%%%%%%%%%%%%%%%%%%%%%%%%%%%%%%%%%%%%%%%%%%%%%%%%%%%%%
    \section{Notations}
    \label[appendix]{appx: notation}
    %%%%%%%%%%%%%%%%%%%%%%%%%%%%%%%%%%%%%%%%%%%%%%%%%%%%%%%%%%%%%%%%%%%%

    For readability, \cref{tab: notation} lists all notations used in this work.
    \input{assets/tables/notation}

    %%%%%%%%%%%%%%%%%%%%%%%%%%%%%%%%%%%%%%%%%%%%%%%%%%%%%%%%%%%%%%%%%%%%
    \section{Proofs}
    \label[appendix]{appx: theorem-proof}
    %%%%%%%%%%%%%%%%%%%%%%%%%%%%%%%%%%%%%%%%%%%%%%%%%%%%%%%%%%%%%%%%%%%%

    %%%%%%%%%%%%%%%%%%%%%%%%%%%%%%%%%%%%%%%%%%%%%%%%%%%%%%%%%%%%%%%%%%%%
    \subsection{Proof of Lemma 1}
    \label[appendix]{appx: first}
    %%%%%%%%%%%%%%%%%%%%%%%%%%%%%%%%%%%%%%%%%%%%%%%%%%%%%%%%%%%%%%%%%%%%

    \begin{lemma}
        Given vector fields $\vv_{\text{cond}}$ generating conditional probability paths $p(\vx,t\mid\vx_{0})$, for any $\rvx_{0}\!\sim\!p(\rvx_{0})$, the expected conditional velocity field is equal to the marginal velocity field
        \begin{equation*}
            \vv(\vx,t)=\E_{\vx_{0}\sim{p(\vx_{0}\mid\vx_{t}=\vx)}}\!\left[\vv(\vx,t\mid\vx_{0})\right].
        \end{equation*}
    \end{lemma}

    \begin{proof}
        We follow the standard flow-matching construction~\cite{lipman2023flowmatchinggenerativemodeling,tong2024improvinggeneralizingflowbasedgenerative}. A sufficient and necessary condition for a velocity field $\vv(\vx,t)$ to generate a probability density path is given by the continuity equation~\cite{villani2009optimaltransport}, which writes:
        \begin{equation}
            \frac{\partial}{\partial{t}}p(\vx,t)+\mathrm{div}\left(p(\vx,t)\vv(\vx,t)\right)=0.
            \label{eq: continuity-eq}
        \end{equation}
        Meanwhile, by the law of total probability and the definition of the conditional velocity field,
        \begin{equation}
            p(\vx,t)=\int_{\gX}p(\vx,t\mid\vx_{0})\,p(\vx_{0})\,\mathrm{d}\vx_{0},\qquad\forall\,p(\vx_{0}).
        \end{equation}
        Taking the derivative with respect to time step $t$ on both sides and applying Bayes' rule yields
        \begin{equation}
            \begin{aligned}
                \frac{\partial}{\partial{t}}p(\vx,t)
                &=\int_{\gX}\frac{\partial}{\partial{t}}p(\vx,t\mid{\vx_{0}})\,p(\vx_{0})\mathrm{d}\vx_{0}\\
                &=\int_{\gX}-\mathrm{div}\left(p(\vx,t\mid{\vx_{0}})\vv(\vx,t\mid{\vx_{0}})\right)\cdot p(\vx_{0})\mathrm{d}\vx_{0}\\
                &=-\mathrm{div}\cdot\left(\int_{\gX}\vv(\vx,t\mid{\vx_{0}})\!\cdot\!p(\vx,t)\!\cdot\!p(\vx_{0}\mid{\vx_{t}=\vx})\mathrm{d}\boldsymbol{x}_{0}\right)\\
                &=-\mathrm{div}\left(p(\vx,t)\E_{\vx_{0}\sim{p(\vx_{0}\mid\vx_{t}=\vx)}}\!\left[\vv(\vx,t\mid\vx_{0})\right]\right).
            \end{aligned}
            \label{eq: derivative-of-marginal-prob}
        \end{equation}
        Combining~\eqref{eq: continuity-eq} and~\eqref{eq: derivative-of-marginal-prob} recovers the relationship. Proof completes.
    \end{proof}

    \subsection{Proof of~\Cref{thm: jacobi-variance} (Jacobian Variance Amplification)}

    \begin{theorem}[Jacobian Variance Amplification]
        Let $\vg\!\triangleq\!\nabla_{\vtheta}\vu_{\vtheta}(\vx_{t},r,t)\!\in\!\mathbb{R}^{d\times p}$ be the parameter Jacobian of the average velocity. The trace of the conditional gradient covariance (\textit{i.e., the total variance}) of $\ell_{\text{MF}}(\vtheta)$ in~\eqref{eq: mf-loss} satisfies
        \begin{equation*}
            \boxed{
                \Tr\!\left(\Cov\!\left[\nabla_{\vtheta}\ell_{\text{MF}}\mid\vx_{t}\right]\right)
                \;\propto\;
                \Tr\!\left(\vg^{\!\top}\,\mJ\,\Sigma_{\vv^{\prime}}\,\mJ^{\!\top}\,\vg\right).
            }
        \end{equation*}
    \end{theorem}

    \begin{proof}
        By substituting $\vv_{\text{cond}}\!=\!\vv(\vx,t)+\vv^{\prime}$, the per-step loss $\ell_{\text{MF}}$ expands:
        \begin{equation}
            \begin{aligned}
            \ell_{\text{MF}} &=\left\|\vu_{\vtheta}+(t-r)\,\texttt{sg}\!\left[\texttt{JVP}(\vu_{\vtheta},(\vx_{t},r,t),(\vv_{\text{cond}},0,1))\right]-\vv_{\text{cond}}\right\|_{2}^{2} \\
            &= \left\|\vu_{\vtheta}+(t-r)\texttt{sg}\left[\partial_{\vx_{t}}\vu_{\vtheta}\!\cdot\!\vv(\vx,t)+\partial_{t}\vu_{\vtheta}\right]-\vv(\vx,t)+\left[(t-r)\partial_{\vx_{t}}\vu_{\vtheta}-\mI_{d}\right]\vv^{\prime}\right\|_{2}^{2} \\
            &= \left\|\vr^{\texttt{sg}}_{\vtheta}+\mJ\vv^{\prime}\right\|_{2}^{2},
            \end{aligned}
        \end{equation}
        where $\vr_{\vtheta}^{\texttt{sg}}$ is the mean-field residual and $\mJ=(t-r)\partial_{\vx_{t}}\vu_{\vtheta}-\mI_{d}$ is the Jacobi factor. With the stop-gradient operator, the gradient of the per-sample loss passes through only the mean-field residual $\vr_{\vtheta}^{\texttt{sg}}$. Let $\vg\!\triangleq\!\nabla_{\vtheta}\vu_{\vtheta}\in\mathbb{R}^{d\times p}$ denote the parameter Jacobian of the network output. The per-step gradient is then given by
        \begin{equation}
            \nabla_{\vtheta}\ell_{\text{MF}}=2\vg^{\!\top}\!\left(\vr_{\vtheta}^{\texttt{sg}}+\mJ\vv^{\prime}\right),
            \label{eq: per-step-grad}
        \end{equation}
        where all three quantities ($\vg$, $\vr_{\vtheta}^{\texttt{sg}}$, and $\mJ$) are deterministic conditioned on $\vx_{t}$. Since $\E_{\vx_{0}\mid\vx_{t}}[\vv^{\prime}]=\bm{0}$ by~\eqref{eq: margv-condv-relation}, the conditional mean gradient then writes
        \begin{equation}
            \E\!\left[\nabla_{\vtheta}\ell\mid\vx_{t}\right]=2\vg^{\!\top}\left(\vr_{\vtheta}^{\texttt{sg}}+\bm{0}\right)=2\vg^{\!\top}\vr^{\texttt{sg}}_{\vtheta}.
        \end{equation}
        The centered gradient is $\nabla_{\vtheta}\ell-\E[\nabla_{\vtheta}\ell\mid\vx_{t}]=2(\nabla_{\vtheta}\vu_{\vtheta})^{\!\top}\mJ\vv^{\prime}$, so the conditional variance is:
        \begin{align}
            \Var\!\left[\nabla_{\vtheta}\ell\mid\vx_{t}\right]
            &=\E\!\left[\left\|2(\nabla_{\vtheta}\vu_{\vtheta})^{\!\top}\mJ\vv^{\prime}\right\|^{2}\mid\vx_{t}\right] \nonumber\\
            &=4\,\E\!\left[(\vv^{\prime})^{\!\top}\mJ^{\!\top}\underbrace{(\nabla_{\vtheta}\vu_{\vtheta})(\nabla_{\vtheta}\vu_{\vtheta})^{\!\top}}_{\mG_{\vtheta}\;\succeq\;\bm{0}}\mJ\vv^{\prime}\;\middle|\;\vx_{t}\right] \nonumber\\
            &=4\,\Tr\!\left(\mG_{\vtheta}\,\mJ\,\Sigma_{\vv^{\prime}}\,\mJ^{\!\top}\right)\;\propto\;\Tr\!\left(\mJ\,\Sigma_{\vv^{\prime}}\,\mJ^{\!\top}\right),
        \end{align}
        where $\mG_{\vtheta}=(\nabla_{\vtheta}\vu_{\vtheta})(\nabla_{\vtheta}\vu_{\vtheta})^{\!\top}\in\mathbb{R}^{d\times d}$ is the parameter-space Gram matrix and the last step uses $\mG_{\vtheta}\succeq\bm{0}$ to establish proportionality using the cyclic property of the trace. Proof completes.
    \end{proof}

    \subsection{Proof of~\Cref{thm: semi-gradient-gap} (Semi-Gradient Gap)}

    \begin{theorem}[Semi-Gradient Gap]
        The gradient of the MeanFlow loss $\gL_{\text{MF}}$ with and without the stop-gradient operator differ by
        \begin{equation*}
            \boxed{
                \underbrace{2(t{-}r)\,\E\!\left[\bigl(\nabla_{\vtheta}\!\left(\partial_{\vx_{t}}\vu_{\vtheta}\!\cdot\!\vv+\partial_{t}\vu_{\vtheta}\right)\bigr)^{\!\top}\vr_{\vtheta}\right]}_{\text{mean-field gradient difference}}
                +\underbrace{\E\!\left[\nabla_{\vtheta}\Tr\!\left(\mJ\Sigma_{\vv^{\prime}}\mJ^{\!\top}\right)\right]}_{\text{variance-driven gradient difference}}.
            }
        \end{equation*}
    \end{theorem}

    \begin{proof}
        The MeanFlow loss conditioned on $\vx_{t}$ decomposes as (cf.~\eqref{eq: per-sample-expand}):
        \begin{equation}
            \E_{\vx_{0}\mid\vx_{t}}\!\left[\ell_{\text{MF}}\right]=\left\|\vr_{\vtheta}^{\texttt{sg}}\right\|^{2}+\Tr\!\left(\mJ\,\Sigma_{\vv^{\prime}}\,\mJ^{\!\top}\right).
            \label{eq: cond-loss-decomp}
        \end{equation}
        \textbf{Without stop-gradient.} Both terms depend on $\vtheta$ (through $\vu_{\vtheta}$ and $\partial_{\vx_{t}}\vu_{\vtheta}$), so the full gradient is:
        \begin{equation}
            \nabla_{\vtheta}\E\!\left[\ell\right]=\nabla_{\vtheta}\left\|\vr_{\vtheta}\right\|^{2}+\nabla_{\vtheta}\Tr\!\left(\mJ\,\Sigma_{\vv^{\prime}}\,\mJ^{\!\top}\right).
        \end{equation}
        \textbf{With stop-gradient.} The JVP terms in $\vr_{\vtheta}^{\texttt{sg}}$ and $\mJ$ are treated as constants, so:
        \begin{equation}
            \nabla_{\vtheta}^{\texttt{sg}}\left\|\vr^{\texttt{sg}}\right\|^{2}=2\left(\nabla_{\vtheta}\vu_{\vtheta}\right)^{\!\top}\vr^{\texttt{sg}},\qquad
            \nabla_{\vtheta}^{\texttt{sg}}\Tr\!\left(\mJ\,\Sigma_{\vv^{\prime}}\,\mJ^{\!\top}\right)=\bm{0}.
        \end{equation}
        The gap between the full and stop-gradient gradients is therefore:
        \begin{align}
            &\nabla_{\vtheta}\E\!\left[\ell\right]-\nabla_{\vtheta}^{\texttt{sg}}\E\!\left[\ell\right] \nonumber\\
            &\quad=\underbrace{\nabla_{\vtheta}\left\|\vr_{\vtheta}\right\|^{2}-2\left(\nabla_{\vtheta}\vu_{\vtheta}\right)^{\!\top}\vr^{\texttt{sg}}}_{\text{mean-field gradient difference}}+\underbrace{\nabla_{\vtheta}\Tr\!\left(\mJ\,\Sigma_{\vv^{\prime}}\,\mJ^{\!\top}\right)}_{\text{variance-driven difference}}.
        \end{align}
        Expanding the first term: $\nabla_{\vtheta}\|\vr\|^{2}=2(\nabla_{\vtheta}\vr_{\vtheta})^{\!\top}\vr_{\vtheta}$, and $\nabla_{\vtheta}\vr_{\vtheta}$ includes the derivative of the JVP terms $(t-r)(\partial_{\vx_{t}}\vu_{\vtheta}\cdot\vv+\partial_{t}\vu_{\vtheta})$ w.r.t.\ $\vtheta$, yielding the second-order mean-field difference $2(t-r)\E[(\nabla_{\vtheta}(\partial_{\vx_{t}}\vu_{\vtheta}\cdot\vv+\partial_{t}\vu_{\vtheta}))^{\!\top}\vr_{\vtheta}]$. At convergence $\vr_{\vtheta}\to\bm{0}$, this term vanishes, and the gap is dominated by the variance-driven term $\nabla_{\vtheta}\Tr(\mJ\,\Sigma_{\vv^{\prime}}\,\mJ^{\!\top})$. Proof completes.
    \end{proof}

    \subsection{Proof of~\Cref{thm: control-variate-optimum} (Optimal Control-Variate Coefficient)}

    \begin{theorem}[Optimal control-variate coefficient]
        Under the scalar-isotropic approximation $\Sigma_{\vv^{\prime}}\!\approx\!\sigma^{2}\mI_{d}$, $\mJ\!\approx\!\kappa\mI_{d}$, and the parameter-isotropy approximation $\vg\vg^{\!\top}\!\propto\!\mI_{d}$,
        \begin{equation*}
            M(\beta) \;\propto\; \beta^{2}(\kappa{+}1)^{2}\|\vb\|^{2} + \sigma^{2}d\bigl((1{-}\beta)\kappa - \beta\bigr)^{2}.
        \end{equation*}
        For $\kappa\!>\!0$, $M(\beta)$ admits a unique minimizer
        \begin{equation*}
            \boxed{
                \beta^{\ast} \;=\; \underbrace{\frac{\kappa}{\kappa+1}}_{\text{noise-cancellation}}\;\cdot\;\underbrace{\frac{\sigma^{2}d}{\sigma^{2}d + \|\vb\|^{2}}}_{\text{shrinkage}},
            }
        \end{equation*}
        with optimum value $M(\beta^{\ast})\!\propto\!\sigma^{2}d\,\kappa^{2}\,\|\vb\|^{2}/(\sigma^{2}d+\|\vb\|^{2})$.
    \end{theorem}

    \begin{proof}
        From~\eqref{eq: g-beta}, we can decompose gradient with coefficient $\beta$ into:
        \begin{equation}
            \vg^{(\beta)} \;=\; 2\,\vg^{\!\top}\!\Bigl[\vr_{\vtheta} \;+\; \bigl((1{-}\beta)\,\mJ - \beta\,\mI\bigr)\,\vv^{\prime} \;+\; \beta\,(\mJ{+}\mI)\,\vb\Bigr].
        \end{equation}
        Using the property $\E_{\vx_{0}\mid\vx_{t}}[\vv^{\prime}]\!=\!\bm{0}$, the conditional mean is $\E[\vg^{(\beta)}\mid\vx_{t}]\!=\!2\vg^{\!\top}(\vr_{\vtheta}+\beta(\mJ{+}\mI)\vb)$ and the centered noise is $-2\vg^{\!\top}\bigl((1{-}\beta)\mJ-\beta\mI\bigr)\vv^{\prime}$ (relative to the ideal target $2\vg^{\!\top}\vr_{\vtheta}$). The conditional MSE w.r.t.\ the ideal target therefore decomposes as
        \begin{equation}
            \mathrm{MSE} \;=\; \underbrace{4\beta^{2}\,\bigl\|\vg^{\!\top}(\mJ{+}\mI)\vb\bigr\|^{2}}_{\text{bias}^{2}} \;+\; \underbrace{4\,\Tr\!\Bigl(\vg\vg^{\!\top}\bigl((1{-}\beta)\mJ-\beta\mI\bigr)\Sigma_{\vv^{\prime}}\bigl((1{-}\beta)\mJ-\beta\mI\bigr)^{\!\top}\Bigr)}_{\text{variance}}.
        \end{equation}
        We then solve for the MSE and prove its strict convexity. Let $\mA\!\triangleq\!\mJ{+}\mI$ and $\mG_{\vtheta}\!\triangleq\!\vg\vg^{\!\top}\!\succeq\!\bm{0}$. Substituting $((1{-}\beta)\mJ-\beta\mI)\!=\!\mJ-\beta\mA$ and using $\|\vg^{\!\top}\mA\vb\|^{2}\!=\!\Tr(\mG_{\vtheta}\mA\vb\vb^{\!\top}\mA^{\!\top})$, the conditional MSE writes as a quadratic in $\beta$:
        \begin{equation}
            M(\beta) \;=\; \alpha_{0} \;-\; 2\,\alpha_{1}\,\beta \;+\; \alpha_{2}\,\beta^{2},
            \label{eq: m-beta-quadratic}
        \end{equation}
        with
        \begin{align*}
            \alpha_{0} &\!\triangleq\!\Tr\!\bigl(\mG_{\vtheta}\,\mJ\,\Sigma_{\vv^{\prime}}\,\mJ^{\!\top}\bigr),\\
            \alpha_{1} &\!\triangleq\!\Tr\!\bigl(\mG_{\vtheta}\,\mJ\,\Sigma_{\vv^{\prime}}\,\mA^{\!\top}\bigr),\\
            \alpha_{2} &\!\triangleq\!\Tr\!\bigl(\mG_{\vtheta}\,\mA\,(\Sigma_{\vv^{\prime}}+\vb\vb^{\!\top})\,\mA^{\!\top}\bigr).
        \end{align*}
        We exploit the symmetry of $\mG_{\vtheta}$ and $\Sigma_{\vv^{\prime}}$ to consolidate the linear-in-$\beta$ trace into a single $\alpha_{1}$. The Hessian is then positive:
        \begin{equation}
            \frac{\partial^{2}}{\partial \beta^{2}}M(\beta) \;=\; 2\alpha_{2} \;=\; 2\,\Tr\!\bigl(\mG_{\vtheta}\,\mA\,(\Sigma_{\vv^{\prime}}+\vb\vb^{\!\top})\,\mA^{\!\top}\bigr) \;\geq\; 0,
        \end{equation}
        since each factor inside the trace is positive semi-definite (\emph{i.e., $\mG_{\vtheta}\!\succeq\!\bm{0}$, $\Sigma_{\vv^{\prime}}+\vb\vb^{\!\top}\!\succeq\!\bm{0}$, and $\mA(\Sigma_{\vv^{\prime}}+\vb\vb^{\!\top})\mA^{\!\top}\!=\!\mM\mM^{\!\top}\!\succeq\!\bm{0}$ for $\mM\!\triangleq\!\mA(\Sigma_{\vv^{\prime}}+\vb\vb^{\!\top})^{1/2}$}). Therefore, $M(\beta)$ is convex in $\beta$. However, the Hessian is strictly positive \emph{iff} $\mG_{\vtheta}^{1/2}\mM\!\neq\!\bm{0}$, which holds outside two pathological regimes: (i)~$\mA\!=\!\bm{0}$, or equivalently $\mJ\!=\!-\mI$, in which case the linear coefficient $\alpha_{1}$ also vanishes and $M(\beta)$ is constant; (ii)~$\vg\!=\!\bm{0}$ (\emph{e.g., frozen parameters}). Otherwise, $M(\beta)$ is \emph{strictly} convex with unique unconstrained minimizer
        \begin{equation}
            \beta^{\ast}_{\text{matrix}} \;=\; \frac{\alpha_{1}}{\alpha_{2}} \;=\; \frac{\Tr(\mG_{\vtheta}\,\mJ\,\Sigma_{\vv^{\prime}}\,(\mJ{+}\mI)^{\!\top})}{\Tr(\mG_{\vtheta}\,(\mJ{+}\mI)(\Sigma_{\vv^{\prime}}+\vb\vb^{\!\top})(\mJ{+}\mI)^{\!\top})}.
            \label{eq: beta-star-matrix}
        \end{equation}
        Note that $\beta^{\ast}_{\text{matrix}}$ may fall outside $[0,1]$; the box-constrained optimum on the $\beta\!\in\![0,1]$ interval (cf.~\eqref{eq: tangent-mix}) is then $\mathrm{clip}(\beta^{\ast}_{\text{matrix}},0,1)$, attained at the nearer endpoint by convexity.

        Finally, we derive the scalar approximation. Substituting $\mJ\!=\!\kappa\mI_{d}$, $\Sigma_{\vv^{\prime}}\!=\!\sigma^{2}\mI_{d}$, $\mA\!=\!(\kappa{+}1)\mI_{d}$, and the parameter-isotropy approximation $\mG_{\vtheta}\!\propto\!\mI_{d}$ yields
        \begin{equation*}
            \alpha_{0}\propto\sigma^{2}\kappa^{2}d,\quad
            \alpha_{1}\propto\sigma^{2}\kappa(\kappa{+}1)d,\quad
            \alpha_{2}\propto(\kappa{+}1)^{2}(\sigma^{2}d+\|\vb\|^{2}),
        \end{equation*}
        which equivalently gives
        \begin{equation}
            M(\beta) \;\propto\; \beta^{2}(\kappa{+}1)^{2}\|\vb\|^{2} \;+\; \sigma^{2}d\bigl((1{-}\beta)\kappa - \beta\bigr)^{2}.
        \end{equation}

        Taking the derivative with respect to $\beta$ and setting it to zero gives:
        \begin{equation}
            \beta^{\ast}\!=\!\frac{\alpha_{1}}{\alpha_{2}}\!=\!\frac{\sigma^{2}\kappa(\kappa{+}1)d}{(\kappa{+}1)^{2}(\sigma^{2}d+\|\vb\|^{2})}\!=\!\frac{\kappa}{\kappa{+}1}\cdot\frac{\sigma^{2}d}{\sigma^{2}d+\|\vb\|^{2}}.
        \end{equation}

        At $\beta^{\ast}$ above, we have $(1{-}\beta^{\ast})\kappa - \beta^{\ast} = \kappa\,\|\vb\|^{2}/(\sigma^{2}d+\|\vb\|^{2})$. Substituting it back recovers:
        \begin{equation*}
            M(\beta^{\ast}) \;\propto\; \frac{\sigma^{2}d\,\kappa^{2}\,\|\vb\|^{2}}{\sigma^{2}d+\|\vb\|^{2}}.
        \end{equation*}
        For $\kappa\!>\!0$ and $\|\vb\|^{2}\!>\!0$: $M(\beta^{\ast})/M(0) \!=\! \|\vb\|^{2}/(\sigma^{2}d+\|\vb\|^{2}) \!<\! 1$, so $M(\beta^{\ast})\!<\!M(0)$.
        For the $M(1)$ comparison, $M(1)\!-\!M(\beta^{\ast})\!\propto\!(\kappa{+}1)^{2}\|\vb\|^{2}+\sigma^{2}d-\sigma^{2}d\kappa^{2}\|\vb\|^{2}/(\sigma^{2}d+\|\vb\|^{2})$, and the first two terms exceed $\sigma^{2}d\kappa^{2}\|\vb\|^{2}/(\sigma^{2}d+\|\vb\|^{2})$ since $(\kappa{+}1)^{2}\!>\!\kappa^{2}$ for $\kappa\!>\!0$. Strict dominance is therefore over the unconstrained interior optimum; if $\beta^{\ast}_{\text{matrix}}\!\notin\![0,1]$, the box-constrained optimum is at the nearer corner, and the dominance is non-strict.
    \end{proof}

    \subsection{Proof of~\Cref{prop: vamf-grad} (Gradient under the $\beta\!=\!1$ corner loss)}

    \begin{proposition}[Gradient under the $\beta\!=\!1$ corner loss]
        The gradient of $\gL_{\text{MF}}^{\text{EMA}}$ takes the form
        \begin{equation*}
            \nabla_{\vtheta}\gL_{\text{MF}}^{\text{EMA}}=2\,\E\!\left[\vg\,\tilde{\vr}^{\text{EMA}}\right],
        \end{equation*}
        where $(\nabla_{\vtheta}\vu_{\vtheta})\!\in\!\mathbb{R}^{d\times p}$ is the parameter Jacobian, $\tilde{\vr}^{\text{EMA}}\!\triangleq\!V_{\vtheta}^{\text{EMA}}\!-\!\vv(\vx_{t},t)\!\in\!\mathbb{R}^{d}$, and no $\mJ\vv^{\prime}$ noise term appears.
        \label[proposition]{prop: vamf-grad}
    \end{proposition}

    \begin{proof}
        The EMA velocity anchor $\vu_{\bar{\vtheta}}(\vx_{t},t,t)$ is deterministic given $\vx_{t}$, since the EMA parameters $\bar{\vtheta}$ are fixed during the gradient step. Under stop-gradient, $V_{\vtheta}^{\text{EMA}}$ depends on $\vtheta$ only through $\vu_{\vtheta}(\vx_{t},r,t)$. By the same argument as~\eqref{eq: per-step-grad}, the per-step gradient is:
        \begin{equation}
            \nabla_{\vtheta}\ell=2\left(\nabla_{\vtheta}\vu_{\vtheta}\right)^{\!\top}\!\left(V_{\vtheta}^{\text{EMA}}-\vv_{\text{cond}}\right).
        \end{equation}
        Writing $\vv_{\text{cond}}=\vv(\vx_{t},t)+\vv^{\prime}$ and noting that $\tilde{\vr}^{\text{EMA}}:=V_{\vtheta}^{\text{EMA}}-\vv(\vx_{t},t)$ is deterministic given $\vx_{t}$:
        \begin{equation}
            \E_{\vx_{0}\mid\vx_{t}}\!\left[\nabla_{\vtheta}\ell\right]=2\left(\nabla_{\vtheta}\vu_{\vtheta}\right)^{\!\top}\tilde{\vr}^{\text{EMA}},
        \end{equation}
        since $\E[\vv^{\prime}\mid\vx_{t}]=\bm{0}$. The tower property $\nabla_{\vtheta}\gL_{\text{MF}}^{\text{EMA}}=\E_{r,t,\vx_{t}}\!\left[\E_{\vx_{0}\mid\vx_{t}}[\nabla_{\vtheta}\ell]\right]$ gives the stated result. No $\mJ\vv^{\prime}$ term appears here because the EMA tangent is deterministic; compare with the original MeanFlow gradient~\eqref{eq: per-step-grad}, where the stochastic tangent produces the $\mJ\vv^{\prime}$ noise.
    \end{proof}

    \subsection{Proof of~\Cref{prop: bias-var} (Bias-Variance Tradeoff)}

    \begin{proposition}[Bias-Variance Tradeoff]
        The EMA mean-field residual decomposes as
        \begin{equation*}
            \tilde{\vr}^{\text{EMA}}=\vr_{\vtheta}+(t{-}r)\,\partial_{\vx_{t}}\vu_{\vtheta}\!\left(\vu_{\bar{\vtheta}}(\vx_{t},t,t)-\vv(\vx_{t},t)\right),
        \end{equation*}
        where $\vr_{\vtheta}$ is the true MeanFlow residual. The FM anchor drives $\vu_{\bar{\vtheta}}(\vx_{t},t,t)\!\to\!\vv(\vx_{t},t)$, giving $\tilde{\vr}^{\text{EMA}}\!\to\!\bm{0}$.
        \label[proposition]{prop: bias-var}
    \end{proposition}

    \begin{proof}
        Expanding $V_{\vtheta}^{\text{EMA}}$ from~\eqref{eq: ema-mf-loss}, the value of the compound prediction (ignoring the stop-gradient, which does not affect values) is:
        \begin{equation}
            V_{\vtheta}^{\text{EMA}}=\vu_{\vtheta}+(t-r)\!\left(\partial_{\vx_{t}}\vu_{\vtheta}\cdot\vu_{\bar{\vtheta}}(\vx_{t},t,t)+\partial_{t}\vu_{\vtheta}\right).
        \end{equation}
        The true MeanFlow residual (with the marginal velocity as tangent) is:
        \begin{equation}
            \vr_{\vtheta}=\vu_{\vtheta}+(t-r)\!\left(\partial_{\vx_{t}}\vu_{\vtheta}\cdot\vv(\vx_{t},t)+\partial_{t}\vu_{\vtheta}\right)-\vv(\vx_{t},t).
        \end{equation}
        Subtracting:
        \begin{align}
            \tilde{\vr}^{\text{EMA}}&=V_{\vtheta}^{\text{EMA}}-\vv(\vx_{t},t) \nonumber\\
            &=\vr_{\vtheta}+(t-r)\,\partial_{\vx_{t}}\vu_{\vtheta}\!\left(\vu_{\bar{\vtheta}}(\vx_{t},t,t)-\vv(\vx_{t},t)\right).
        \end{align}
        At convergence, the true residual $\vr_{\vtheta}\to\bm{0}$ (the MeanFlow identity is satisfied) and the FM anchor loss~\eqref{eq: fm-anchor} supervises the boundary condition, driving $\vu_{\bar{\vtheta}}(\vx_{t},t,t)\to\vv(\vx_{t},t)$ and hence $\tilde{\vr}^{\text{EMA}}\to\bm{0}$.
    \end{proof}

    \subsection{Proof of~\Cref{prop: target-tangent-asymmetry} (Target/Tangent Bias Asymmetry)}

    \begin{proposition}[Target/Tangent Bias Asymmetry]
        Let $\hat{\vv}=\vu_{\bar{\vtheta}}(\vx_{t},t,t)$ be a deterministic proxy with bias $\vb\!\triangleq\!\hat{\vv}-\vv(\vx_{t},t)$. Consider two ways of substituting $\hat{\vv}$ into~\eqref{eq: mf-loss}:
        \begin{enumerate}
            \item[\textnormal{(i)}] \textnormal{Tangent replacement (the $\beta\!=\!1$ corner):} keep the target $\vv_{\text{cond}}$ but replace the JVP tangent. At the loss minimum,
            \begin{equation*}
                \vu^{\ast}(\vx_{t},r,t)-\vv(\vx_{t},t)=-\tfrac{1}{2}(t{-}r)\,\partial_{\vx_{t}}\vu^{\ast}\,\vb+\gO((t{-}r)\,\|\vb\|^{2}),
            \end{equation*}
            so the stationary error is $\gO((t{-}r)\,\|\partial_{\vx_{t}}\vu\|\,\|\vb\|)$ and vanishes at $r\!\to\!t$.
            \item[\textnormal{(ii)}] \textnormal{Target replacement:} substitute $\hat{\vv}$ for $\vv_{\text{cond}}$ in the regression target. At $r\!=\!t$ the loss minimum satisfies
            \begin{equation*}
                \vu^{\ast}(\vx_{t},t,t)-\vv(\vx_{t},t)=\vb,
            \end{equation*}
            i.e., the stationary error equals the proxy bias \emph{exactly}, uniformly in $t$ and independent of the gap $(t{-}r)$.
        \end{enumerate}
        \label[proposition]{prop: target-tangent-asymmetry}
    \end{proposition}

    \begin{proof}
        \textbf{Tangent replacement.} Substituting $\hat{\vv}$ for $\vv_{\text{cond}}$ only in the JVP tangent (target unchanged) gives the per-sample compound prediction
        \begin{equation*}
            V_{\vtheta}^{\text{EMA}}=\vu_{\vtheta}+(t-r)\!\left(\partial_{\vx_{t}}\vu_{\vtheta}\cdot\hat{\vv}+\partial_{t}\vu_{\vtheta}\right),
        \end{equation*}
        and the loss is $\E\|V_{\vtheta}^{\text{EMA}}-\vv_{\text{cond}}\|^{2}$. By stationarity, $\E[V_{\vtheta}^{\text{EMA}}\!-\!\vv_{\text{cond}}\mid\vx_{t}]\!=\!\bm{0}$ at the minimum. Taking conditional expectation and using $\E[\vv_{\text{cond}}\mid\vx_{t}]\!=\!\vv(\vx_{t},t)$,
        \begin{equation*}
            \vu^{\ast}+(t-r)\!\left(\partial_{\vx_{t}}\vu^{\ast}\cdot\hat{\vv}+\partial_{t}\vu^{\ast}\right)=\vv(\vx_{t},t).
        \end{equation*}
        At the same point, the \emph{true} MeanFlow identity (with the marginal velocity as tangent) reads
        \begin{equation*}
            \vu^{\diamond}+(t-r)\!\left(\partial_{\vx_{t}}\vu^{\diamond}\cdot\vv+\partial_{t}\vu^{\diamond}\right)=\vv(\vx_{t},t),
        \end{equation*}
        with $\vu^{\diamond}(\vx_{t},t,t)=\vv(\vx_{t},t)$. Subtracting and writing $\vu^{\ast}=\vu^{\diamond}+\Delta$,
        \begin{equation*}
            \Delta+(t-r)\,\partial_{\vx_{t}}\vu^{\ast}\,\vb+(t-r)\,\mathrm{D}_{t}\Delta=\bm{0}, \qquad \mathrm{D}_{t}\Delta\;\triangleq\;\partial_{t}\Delta+v\!\cdot\!\partial_{\vx_{t}}\Delta,
        \end{equation*}
        with the boundary condition $\Delta(\vx_{t},t,t)\!=\!\bm{0}$. Substituting $s\!=\!t{-}r$ at fixed $r$ converts this to a first-order linear ODE in $s$:
        \begin{equation*}
            s\,\frac{\mathrm{d}\Delta}{\mathrm{d}s}+\Delta=-s\,\partial_{\vx_{t}}\vu^{\ast}\,\vb+\gO(\|\vb\|^{2}),
        \end{equation*}
        which has $s\Delta(s)\!=\!-\tfrac{s^{2}}{2}\partial_{\vx_{t}}\vu^{\ast}\,\vb+C$; the boundary condition $\Delta(0)\!=\!\bm{0}$ forces $C\!=\!0$. Hence
        \begin{equation*}
            \Delta(\vx_{t},r,t)=-\tfrac{1}{2}(t-r)\,\partial_{\vx_{t}}\vu^{\ast}\,\vb+\gO((t{-}r)\,\|\vb\|^{2}),
        \end{equation*}
        i.e., the stationary error is $\gO((t-r)\,\|\partial_{\vx_{t}}\vu\|\,\|\vb\|)$ and vanishes at the boundary $r\!\to\!t$. The factor of $\tfrac{1}{2}$ comes from integrating the material-derivative term, which contributes equally to the leading-order coefficient as the bias source term.
        \\[2pt]
        \textbf{Target replacement.} Substituting $\hat{\vv}$ for $\vv_{\text{cond}}$ in the regression target yields the loss $\E\|\vu_{\vtheta}+(t-r)\,\mathrm{D}\vu/\mathrm{D}t-\hat{\vv}\|^{2}$. By stationarity, the loss minimum satisfies $\E[\vu^{\ast}+(t-r)\,\mathrm{D}\vu^{\ast}/\mathrm{D}t-\hat{\vv}\mid\vx_{t}]\!=\!\bm{0}$. Specializing to $r\!=\!t$, the gap term vanishes and the equation collapses to $\vu^{\ast}(\vx_{t},t,t)\!=\!\hat{\vv}\!=\!\vv(\vx_{t},t)+\vb$, so the stationary error at the boundary is exactly $\vb$, independent of $(t,r)$.
        \\[2pt]
        \textbf{Asymmetry.} Tangent replacement carries a multiplicative $(t-r)$ factor on the bias and is therefore pinned to zero at the boundary $r\!=\!t$ that the FM anchor enforces; target replacement carries no such factor, and the bias persists at every $t$, requiring a boundary regularizer at \emph{every} $(r,t)$ to remove.
    \end{proof}

    %%%%%%%%%%%%%%%%%%%%%%%%%%%%%%%%%%%%%%%%%%%%%%%%%%%%%%%%%%%%%%%%%%%%
    \section{Implementation Details}
    \label[appendix]{appx: implementation}
    %%%%%%%%%%%%%%%%%%%%%%%%%%%%%%%%%%%%%%%%%%%%%%%%%%%%%%%%%%%%%%%%%%%%

    To validate the framework empirically in ~\cref{sec: experiment}, we instantiate the $\beta\!=\!1$ corner of~\cref{thm: control-variate-optimum} through two design choices at near-zero compute overhead:

    \noindent\textbf{EMA proxy.} We take $\hat{\vv}\!=\!\vu_{\bar{\vtheta}}(\vx_{t},t,t)$ where $\bar{\vtheta}$ is a Polyak-averaged~\cite{polyak1992acceleration,tarvainen2017mean} (\emph{a.k.a. exponential moving-average}) copy of $\vtheta$, where $\bar{\vtheta}\!\gets\!\mu\bar{\vtheta}\!+\!(1{-}\mu)\vtheta$ with $\mu\!\in\![0,1)$. The resulting loss replaces the conditional tangent in~\eqref{eq: mf-loss} with the EMA proxy:
    \begin{equation}
        \mathcal{L}_{\text{MF}}^{\text{EMA}}(\vtheta)=
        \E_{r,t,\vx_{0},\vx_{1}}\!\left[\left\|\vu_{\vtheta}+(t{-}r)\,\texttt{sg}\!\left[\texttt{JVP}\!\bigl(\vu_{\vtheta},(\vx_{t},r,t),(\vu_{\bar{\vtheta},t},0,1)\bigr)\right]-\vv_{\text{cond}}\right\|_{2}^{2}\right],
        \label{eq: ema-mf-loss}
    \end{equation}
    with $\vu_{\bar{\vtheta},t}\!\triangleq\!\vu(\vx_{t},t,t;\bar{\vtheta})$. The regression target remains $\vv_{\text{cond}}$: the tangent and target play distinct statistical roles per~\cref{subsec: two-roles}, and we only need to replace the tangent.

    \noindent\textbf{Flow-matching anchor.} The proxy bias $\vb\!=\!\vu_{\bar{\vtheta},t}\!-\!\vv$ must be controlled or it propagates through the $\beta\,(\mJ{+}\mI)\vb$ bias term. We supervise the boundary condition $\vu(\vx_{t},t,t)\!=\!\vv(\vx_{t},t)$ directly with a small flow-matching loss
    \begin{equation}
        \gL_{\text{FM}}(\vtheta)=\E_{t,\vx_{0},\vx_{1}}\!\left[\left\|\vu_{\vtheta}(\vx_{t},t{-}\delta,t)-\vv_{\text{cond}}\right\|_{2}^{2}\right],
        \label{eq: fm-anchor}
    \end{equation}
    where $\delta\!\sim\!\gU[\delta_{\min},\delta_{\max}]$ with $\delta_{\max}\!\ll\!1$. Here, we evaluate the anchor at a small \emph{interior} time offset rather than at the boundary $\delta\!=\!0$ for two reasons. First, the same $\vu(\vx_{t},t{-}\delta,t)$ is what the EMA tangent eventually converges to, so the anchor and the JVP target share the same network-evaluation pattern. Meanwhile, we found training to be slightly more stable when the anchor probes a thin neighborhood of the boundary rather than a single $r\!=\!t$ slice. For small $\delta$, $\vu(\vx_{t},t{-}\delta,t)$ approximates the instantaneous velocity, and~\eqref{eq: margv-condv-relation} guarantees $\E[\vv_{\text{cond}}\mid\vx_{t}]\!=\!\vv(\vx_{t},t)$, so $\gL_{\text{FM}}$ drives $\vu_{\vtheta}(\vx_{t},t,t)\!\to\!\vv(\vx_{t},t)$, and via EMA, $\vu_{\bar{\vtheta}}(\vx_{t},t,t)\!\to\!\vv$ and hence $\vb\!\to\!\bm{0}$. The full $\beta\!=\!1$ corner loss is
    \begin{equation}
        \gL_{\beta=1}(\vtheta)=\gL_{\text{MF}}^{\text{EMA}}(\vtheta)+\lambda\,\gL_{\text{FM}}(\vtheta),
        \label{eq: vamf-obj}
    \end{equation}
    with $\lambda\!>\!0$. The training procedure is~\cref{alg: va-mf}. For the hyperparameters in our experiments, we use $(\delta_{\min},\delta_{\max},\lambda)\!=\!(10^{-4},10^{-2},0.5) $for 2-D toy and DGMM runs, and $(\delta_{\min},\delta_{\max},\lambda)\!=\!(0,10^{-3},0.1)$ for training the DiT-B/4 on ImageNet-$256$ with $\beta\!=\!1$.

    \input{assets/algorithms/vamf}

    \noindent\textbf{Optimality.} The combination of EMA proxy plus FM anchor approximates the $\beta\!\to\!1$ limit of~\cref{thm: control-variate-optimum} at the cost of one additional JVP per step.~\Cref{prop: vamf-grad,prop: bias-var,prop: target-tangent-asymmetry} (stated and proved in~\cref{appx: theorem-proof}) characterize the recipe theoretically: the $\mJ\vv^{\prime}$ amplification of~\cref{thm: jacobi-variance} is eliminated under the EMA tangent (\Cref{prop: vamf-grad}); the residual EMA bias is controlled by the FM anchor (\Cref{prop: bias-var}); and a target/tangent asymmetry justifies why the FM anchor at the boundary $r\!=\!t$ alone is sufficient since the tangent-replacement bias is multiplicatively damped by $(t{-}r)$, while target-replacement bias is not (\Cref{prop: target-tangent-asymmetry}). The same asymmetry locates a failure mode of self-distillation-style variants that use a deterministic proxy for both roles, where the trivial fixed point $\vu^{\ast}\!\equiv\!\hat{\vv}$ persists without continual external supervision~\cite{geng2025improvedmeanflowschallenges,zhang2026overcomingcurvaturebottleneckmeanflow}.

    %%%%%%%%%%%%%%%%%%%%%%%%%%%%%%%%%%%%%%%%%%%%%%%%%%%%%%%%%%%%%%%%%%%%
    \section{Additional Results}
    \label[appendix]{appx: results}
    %%%%%%%%%%%%%%%%%%%%%%%%%%%%%%%%%%%%%%%%%%%%%%%%%%%%%%%%%%%%%%%%%%%%

    \subsection{2-D Toy Benchmark}
    \label[appendix]{appx: toy-details}

    \Cref{tab: toy} reports the matched-method comparison of the original MeanFlow ($\beta\!=\!0$) against the full $\beta\!=\!1$ corner recipe (EMA tangent + FM anchor; hyperparameters in~\cref{appx: implementation}, no trace weight, no adaptive weighting) on six 2-D datasets, averaged over three seeds at $200$k optimization steps. The $\beta\!=\!1$ recipe reduces SW$_{1}$ on five of six datasets, with the largest reduction ($54\%$) on \texttt{swiss\_roll}. \texttt{two\_spirals} is a known low-curvature outlier where the original MeanFlow already attains the lowest absolute SW$_{1}$; the empirical $\beta^{\ast}\!=\!0$ on the $\beta$-sweep agrees.

    \input{assets/tables/toy}

    \subsection{Full $\beta$-sweep grid on DGMM}

    \Cref{tab: dgmm-beta-full} reports SW$_{1}$ at convergence for every $(d,\beta)$ pair, with mean$\pm$SEM over three seeds. The bolded cell per row is the empirical $\beta^{\ast}\!=\!\arg\min_{\beta}\mathrm{SW}_{1}(\beta)$; the summary view in~\cref{tab: dgmm-beta-summary} (main text) compares only $\beta\!=\!0$ and $\beta^{\ast}$ per dimension. The grid makes the $\beta$-insensitivity of low-$d$ rows visible at full precision: SW$_{1}$ is uniform to three decimal places for $d\!\in\!\{4,8,16\}$.

    \input{assets/tables/dgmm-beta-full}

    \subsection{Stability pilot: learning-rate sweep}
    \label{appx: stability-pilot}

    \Cref{thm: control-variate-optimum} predicts that the deterministic EMA tangent ($\beta\!=\!1$) lowers the per-step gradient variance, which might suggest greater tolerance to large learning rates. We test this directly with a learning-rate sweep on \texttt{eight\_gaussians} and \texttt{checkerboard}, training the three-layer MLP for $50{,}000$ steps at each learning rate in $\{10^{-4},\,3{\times}10^{-4},\,10^{-3},\,3{\times}10^{-3},\,10^{-2},\,3{\times}10^{-2}\}$ for $\beta\!\in\!\{0,0.5,1\}$ over two seeds. A run is counted as diverged if its loss reaches NaN/inf or its final $\mathrm{SW}_1$ exceeds $10\times$ the best-learning-rate value at that dataset; the maximum stable learning rate is the top of the unbroken stable range.

    \begin{table}[!t]
        \centering
        \caption{Maximum stable learning rate per dataset and tangent-mixing coefficient $\beta$ in the stability pilot. A cell is stable iff neither seed diverged; the maximum stable learning rate is the top of the unbroken stable range over the grid $\{10^{-4},3{\times}10^{-4},10^{-3},3{\times}10^{-3},10^{-2},3{\times}10^{-2}\}$.}
        \label{tab: stability-pilot}
        \begin{tabular}{@{}lccc@{}}
            \toprule
            \textbf{Dataset} & $\beta\!=\!0$ & $\beta\!=\!0.5$ & $\beta\!=\!1$ \\
            \midrule
            \texttt{eight\_gaussians} & $10^{-2}$ & $10^{-2}$         & $3{\times}10^{-3}$ \\
            \texttt{checkerboard}     & $10^{-2}$ & $3{\times}10^{-3}$ & $3{\times}10^{-3}$ \\
            \bottomrule
        \end{tabular}
    \end{table}

    Contrary to the per-step-variance intuition, the EMA instantiation does not tolerate higher learning rates: on \texttt{eight\_gaussians} the maximum stable learning rate is $10^{-2}/10^{-2}/3{\times}10^{-3}$ for $\beta\!=\!0/0.5/1$, and on \texttt{checkerboard} it is $10^{-2}/3{\times}10^{-3}/3{\times}10^{-3}$; in both cases $\beta\!=\!1$ diverges at a learning rate at or below the vanilla $\beta\!=\!0$ value. The diagnostics show the mechanism: as the learning rate grows, the EMA parameters lag the online parameters, the tangent goes stale, and the noise ratio grows unboundedly before divergence. The result is consistent with our framework, in which we predict that the variance reduction is a per-step property. In contrast, the stale-tangent bias enters multiplicatively through $\beta(\mJ{+}\mI)\vb$. Meanwhile, it scopes our recommendation in~\cref{sec: conclusion}. We note this is a toy/DGMM-scale observation; we did not run a DiT-scale stability sweep.

    \subsection{Latent DiT-B/4 on ImageNet-$256$}
    \label[appendix]{appx: dit-results}

    We test whether the variance-reduction mechanism of~\cref{thm: jacobi-variance,thm: control-variate-optimum} scales beyond toys by training DiT~\cite{peebles2023scalablediffusion} variants on ImageNet in the latent space of a frozen Stable Diffusion VAE, following the recipe of~\cite{geng2025meanflowsonestepgenerative}. The four $\beta$ configurations share identical hyperparameters and code paths; only the loss differs (\emph{i.e., the original MeanFlow MSE for the interior $\beta$ runs versus the EMA-tangent loss in~\eqref{eq: vamf-obj} for the $\beta\!=\!1$ corner}).

    \noindent\textbf{Setup.} Logit-normal $(r,t)$ sampling with mean $-0.4$ and stddev $1$, overlap rate $0.75$, AdamW with constant learning rate $10^{-4}$, EMA decay $0.9999$, batch size $256$ across $32$ TPU v4 chips. All four configurations ($\beta\!\in\!\{0,0.25,0.5,1\}$) share the same DiT-B/4 backbone, data pipeline, optimizer, EMA schedule, and $(r,t)$ sampling distribution. The interior $\beta$ runs ($\beta\!\in\!\{0,0.25,0.5\}$) reuse the baseline configuration verbatim and add only the tangent-mixing rule of~\eqref{eq: tangent-mix}; the $\beta\!=\!1$ corner additionally turns off the Karras-style adaptive loss weighting and adds the FM anchor of~\eqref{eq: fm-anchor} (DiT anchor hyperparameters in~\cref{appx: implementation}). The extra EMA forward pass and FM-anchor evaluation per step add up to a $\sim\!22\%$ wall-clock overhead on TPU v4 ($2.20$ steps/s for $\beta\!=\!1$ versus $2.65$ steps/s for the baseline); the $\beta\!\in\!\{0.25,0.5\}$ interior runs incur only the EMA forward overhead ($\sim\!10\%$).

    \noindent\textbf{Matched-step ordering.} \Cref{tab: dit-fid} reports FID$_{50\text{k}}$ at representative $5$k-step checkpoints for all four runs. The coarse ordering is unambiguous and holds at every matched-step checkpoint after the early-training transient ($\le\!25$k steps): $\mathrm{FID}(\beta\!=\!1)$ sits far above the other three (a $+12\!\sim\!+14$ gap to the baseline across the converged regime), and $\mathrm{FID}(\beta\!=\!0.5)$ stabilizes a clear $+1.0\!\sim\!+1.5$ above the baseline. The fine gap between $\beta\!=\!0$ and $\beta\!=\!0.25$ is much smaller, tightening from $+9.5$ early in training to $+0.39$ at step $295$k. This sub-$0.5$ FID gap is comparable to the across-seed FID variation we observe for the baseline configuration, so we treat the full four-point ordering at every intermediate checkpoint as a \emph{within-run} consistency check rather than a seed-robust separation; an independently seeded baseline can reorder $\beta\!=\!0$ and $\beta\!=\!0.25$ at intermediate steps. At convergence (step $295$k), the four floors nonetheless recover the predicted ordering (below).

    \noindent\textbf{Converged FID floors and four-point ordering.} All four configurations completed their scheduled $300$k-step training and have a matched-step eval at step $295$k. The four FID floors at step $295$k recover the bias-variance ordering of~\cref{thm: control-variate-optimum}:
    \begin{align*}
        \mathrm{FID}(\beta\!=\!0)&\;=\;11.37\\
        \mathrm{FID}(\beta\!=\!0.25)&\;=\;11.76\\
        \mathrm{FID}(\beta\!=\!0.5)&\;=\;12.51\\
        \mathrm{FID}(\beta\!=\!1)&\;=\;23.36\\
        \mathrm{FID}(\beta\!=\!0)\;<\;\mathrm{FID}(\beta\!=\!0.25)&\;<\;\mathrm{FID}(\beta\!=\!0.5)\;<\;\mathrm{FID}(\beta\!=\!1)
    \end{align*}
    Quantitatively, taking offsets from the baseline floor $11.37$, $\Delta\mathrm{FID}(\beta\!=\!0.5)\!\approx\!1.14$ and $\Delta\mathrm{FID}(\beta\!=\!1)\!\approx\!11.99$. Anchoring the gradient-MSE prediction $\beta^{2}\!\cdot\!A$ at the $\beta\!=\!0.5$ point gives $A\!\approx\!4.58$, so the predicted offset at $\beta\!=\!1$ is $4.58$. The empirical $11.99$ exceeds this by $\sim\!2.6\!\times$. The FID landscape is therefore \emph{super-linear} in the MSE-axis bias: it penalizes large bias more aggressively than gradient MSE does, which pulls the FID-optimal $\beta$ past the gradient-MSE interior minimizer to the unbiased corner $\beta\!=\!0$. The large, seed-robust separations: the $\beta\!=\!1$ corner far above the interior runs and the super-linear $\sim\!2.6\!\times$ FID penalty relative to the gradient-MSE prediction. These are the clearest DiT-scale confirmations of the bias-variance decomposition.

    \input{assets/tables/dit-fid}

    \noindent\textbf{Bias and shrinkage trajectory.} To localize the predicted optimum we measure two diagnostics on the DiT checkpoints: an EMA-tracking proxy $\widehat{\|\vb\|^{2}}\!\triangleq\!\E_{\vx_{t}}\bigl[\|\vu_{\vtheta}(\vx_{t},t,t)\!-\!\vu_{\bar{\vtheta}}(\vx_{t},t,t)\|^{2}\bigr]$ (the boundary gap between current parameters and the EMA copy), and the irreducible noise floor $\sigma^{2}d\!\approx\!8.2{\times}10^{3}$ from $\Tr(\Sigma_{\vv^{\prime}})$ on the same batch.~\cref{tab: dit-bias-shrinkage} reports the trajectory: the EMA-tracking proxy \emph{decays} through training, and the shrinkage factor $\sigma^{2}d/(\sigma^{2}d{+}\widehat{\|\vb\|^{2}})$ stays near $1$ throughout converged training. We caveat that $\widehat{\|\vb\|^{2}}$ is the EMA-tracking gap, not the true model bias $\|\vu_{\vtheta}(\vx_{t},t,t)\!-\!\vv(\vx_{t},t)\|^{2}$; cleanly estimating the latter requires access to the conditional distribution $p(\vx_{0}\!\mid\!\vx_{t})$, which is intractable on ImageNet. Combined with $\kappa\!\approx\!1$ from the Frobenius-norm Hutchinson trace, the scalar-isotropic substitution gives $\beta^{\ast}\!\in\![0.46,\,0.50]$ across the full trajectory; this is a coarse first-cut prediction that the direct matrix-form measurement (see~\cref{subsec: exp-dit}) refines to $\beta^{\ast}_{\text{no bias}}\!\approx\!0.94$.

    \input{assets/tables/dit-bias-shrinkage}

    \noindent\textbf{Qualitative samples.} \Cref{fig: dit-qualitative-beta0,fig: dit-qualitative-beta05,fig: dit-qualitative-beta1} show $100$ class-conditional generations from the three available checkpoints at step $300$k, one figure per $\beta$ value. The three figures share the same noise key (seed $7$) and the same $100$ ImageNet class labels (drawn uniformly from $\{0,\dots,999\}$ with an independent label seed of $7$, no curation), so per-position differences across the three figures isolate the effect of the loss alone. The $\beta\!=\!0$ versus $\beta\!=\!1$ contrast is at the perceptual scale, consistent with the $\!\sim\!2\!\times$ FID gap; the intermediate $\beta\!=\!0.5$ figure is at the metric-difference scale and is not always visually separable from the baseline at this sample budget.

    \subsection{Independent validation of the marginal-velocity bias}
    \label{appx: vref}

    The matrix-form estimate of~\cref{subsec: exp-dit} uses the model's own boundary $\vu_{\vtheta}(\vx_{t},t,t)$ as the proxy for the marginal velocity, so the measured bias is referenced to the same network. To obtain a non-circular estimate, we train an independent, unconditional flow-matching model on the same Stable-Diffusion latents and DiT-B/4 backbone.
    %($55$k steps to convergence; all evaluations at the null class, index $1000$, from a pool of $2048$ latents with boundary $r\!=\!t$),
    We use this model purely as a reference $\vv_{\text{ref}}$ for the marginal velocity. To begin with, we confirm the reference is itself low-bias: at $t\!=\!0.1$ its boundary residual is $4165.27$ against the analytic noise floor $4096$ (data dimension $d\!=\!4096$), a $\sim\!1.7\%$ relative bias. We then recompute the bias as $\|\vu_{\vtheta}(\vx_{t},t,t)-\vv_{\text{ref}}(\vx_{t},t)\|^{2}$ at $t\!\in\!\{0.1,0.3,0.5,0.7,0.9\}$ with $1024$ samples per $t$. \Cref{fig: vref} shows the matrix-form $\beta^{\ast}$ and the bias-to-noise ratio per probed $t$.

    The ratio stays in $[0.006, 0.043]$ across $t$ (shrinkage $[0.959, 0.994]$), and the matrix-form $\beta^{\ast}$ recomputed with the independent bias lies in $[0.901, 0.935]$using. The result agrees with the one with the model's own boundary: the bias is small, and $\beta^{\ast}_{\text{matrix}}$ sits near the top of $(0,0.94]$. The independent ratio is larger than the EMA-self-proxy ratio (in $[0.001, 0.004]$ across $t$) but yields the same conclusion. We make the reference model unconditional to match the unconditional MeanFlow setting in our theory, and leave the conditional model, or the one with Classifier-free Guidance, to future work.

    \pagebreak

    {%
        \small
        \putbib[references]
    }

    \newpage
    \begin{figure}[t]
        \centering
        \includegraphics[width=\textwidth]{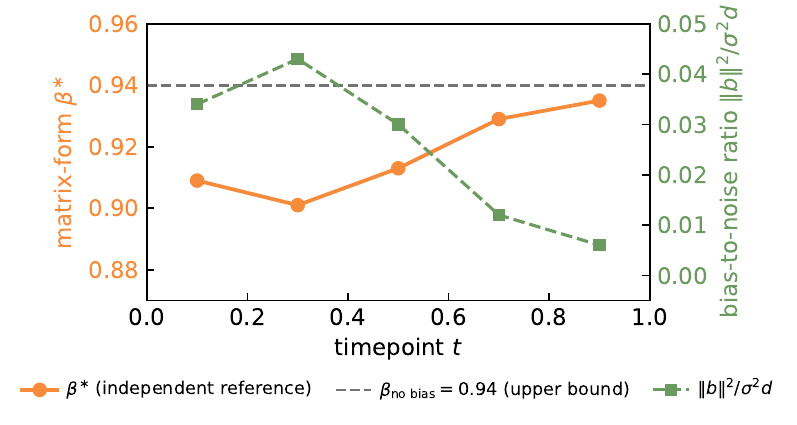}
        \caption{Non-Circular Bias Diagnostics against the Independent Reference $\vv_{\text{ref}}$. Matrix-form $\beta^{\ast}\!=\!\beta_{\text{no bias}}\!\cdot\!\text{shrinkage}$ (left axis, no-bias bound $\beta_{\text{no bias}}\!=\!0.94$ marked) and bias-to-noise ratio $\|\vb\|^{2}/\sigma^{2}d$ (right axis) per probed $t$, $1024$ samples each. The independent $\beta^{\ast}$ stays in $[0.901,0.935]$, just below the no-bias bound, and the ratio remains small ($\le\!0.043$).}
        \label{fig: vref}
    \end{figure}
    \begin{figure}[t]
        \centering
        \includegraphics[width=\textwidth]{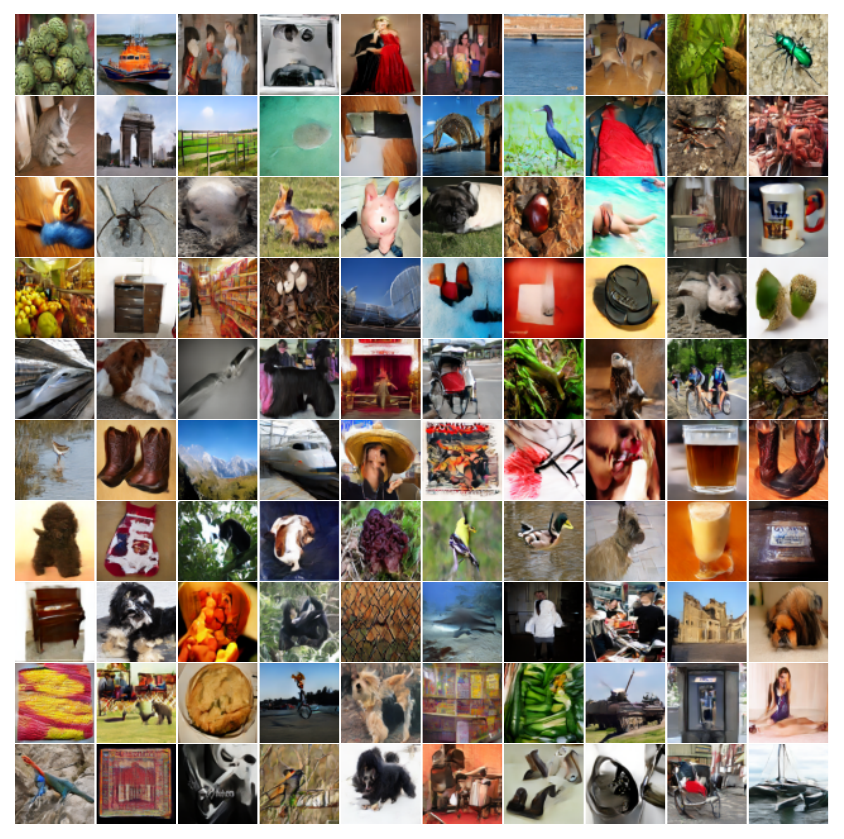}
        \caption{$100$ Class-Conditional Samples from the $\beta\!=\!0$ Baseline Checkpoint at Step $300$k (FID $11.37$). Same noise seed and labels as~\cref{fig: dit-qualitative-beta05,fig: dit-qualitative-beta1}.}
        \label{fig: dit-qualitative-beta0}
    \end{figure}

    \begin{figure}[t]
        \centering
        \includegraphics[width=\textwidth]{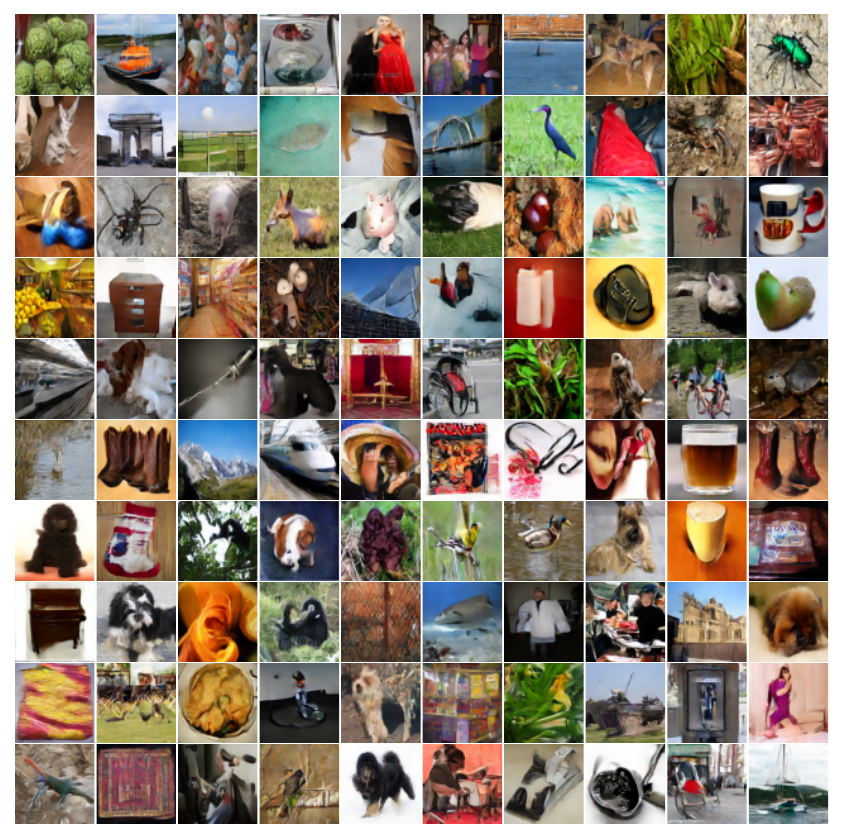}
        \caption{$100$ Class-Conditional Samples from the $\beta\!=\!0.5$ Checkpoint at Step $300$k (FID $12.51$). Same noise seed and labels as~\cref{fig: dit-qualitative-beta0,fig: dit-qualitative-beta1}.}
        \label{fig: dit-qualitative-beta05}
    \end{figure}

    \begin{figure}[t]
        \centering
        \includegraphics[width=\textwidth]{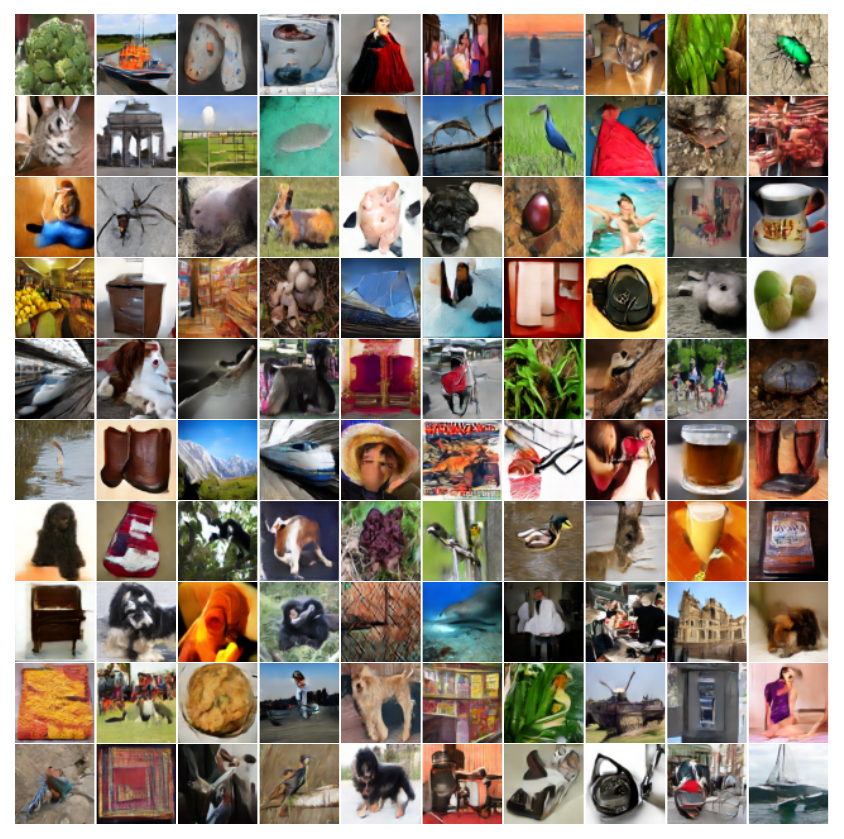}
        \caption{$100$ Class-Conditional Samples from the $\beta\!=\!1$ Corner Checkpoint at Step $300$k (FID $23.36$). Same noise seed and labels as~\cref{fig: dit-qualitative-beta0,fig: dit-qualitative-beta05}.}
        \label{fig: dit-qualitative-beta1}
    \end{figure}

\end{appendices}

%% file: assets/tables/notation.tex
\begin{table}[!ht]
    \caption{Table of notations.}
    \label{tab: notation}
    \centering
    \begin{tabular}{@{}ll@{}}
        \toprule
        \textbf{Notation}  &   \textbf{Description}   \\
        \midrule
        \multicolumn{2}{@{}l}{\emph{Scalars}} \\
        $d$ & Dimension of the data samples \\
        $p$ & Number of parameters in the model \\
        $c$ & Coefficient in the fitted gradient-MSE curve $c\beta^{2}$ \\
        \midrule
        \multicolumn{2}{@{}l}{\emph{State vectors and velocity fields}} \\
        $\rvx$, $\rvx_{t}$  & Random data sample / state at interpolant time $t$ \\
        $\vx$, $\vx_{t}$ & Realized state / state at time $t$\\
        $\vv(\vx,t)$ & Marginal velocity field \\
        $\vv_{\text{cond}}$ & Conditional velocity $\vv_{\text{cond}}\!\triangleq\!\vv(\vx,t|\vx_{0})$.\\
        $\vv^{\prime}$ & Conditional velocity fluctuation: $\vv^{\prime}=\vv_{\text{cond}}-\vv(\vx,t)$ \\
        $\hat{\vv}$ & Deterministic proxy for the marginal velocity field $\vv$ \\
        \midrule
        \multicolumn{2}{@{}l}{\emph{Model}} \\
        $\vu_{\vtheta}(\vx,r,t)$ & Two-parameter average velocity field with parameters $\vtheta$ \\
        $\vu_{\bar{\vtheta}}$ &  Copy of $\vu_{\vtheta}$ with exponential moving-average parameters $\bar{\vtheta}$ \\
        \midrule
        \multicolumn{2}{@{}l}{\emph{Spatial-Jacobian quantities}} \\
        $\partial_{\vx_{t}}\vu_{\vtheta}$ & Spatial Jacobian of $\vu_{\vtheta}$ with respect to state $\vx_{t}$ \\
        $\mJ$ & Jacobi factor: $\mJ=(t-r)\partial_{\vx_{t}}\vu_{\vtheta}-\mI_{d}\in\mathbb{R}^{d\times d}$ \\
        $\mA$ & Shorthand $\mA\!\triangleq\!\mJ+\mI_{d}=(t-r)\partial_{\vx_{t}}\vu_{\vtheta}$ \\
        $\kappa$ & Scalar-isotropic approximation for $\mJ$: $\mJ\!\approx\!\kappa\mI_{d}$ \\
        \midrule
        \multicolumn{2}{@{}l}{\emph{Bias and Variances}} \\
        $\Sigma_{\vv^{\prime}}$ & Covariance matrix of conditional velocity fluctuation: $\Cov_{\vx_{0}\mid\vx_{t}}[\vv^{\prime}]$ \\
        $\sigma^{2}d$ & Total variance of $\vv^{\prime}$ in scalar-isotropic case: $\Sigma_{\vv^{\prime}}\!\approx\!\sigma^{2}\mI_{d}$ \\
        $\vb$ & Proxy bias: $\vb\!\triangleq\!\hat{\vv}-\vv(\vx_{t},t)$ \\
        $\widehat{\|\vb\|_{2}^{2}}$ & EMA-tracking proxy for bias: $\hat{\|\vb\|_{2}^{2}}\!\triangleq\!\E_{\vx_{t}}\!\left[\|\vu_{\vtheta}(\vx_{t},t,t)-\vu_{\bar{\vtheta}}(\vx_{t},t,t)\|^{2}\right]$ \\
        \midrule
        \multicolumn{2}{@{}l}{\emph{Gradients}} \\
        $\nabla_{\vtheta}\ell_{\text{MF}}$ & Per-step gradient of the MeanFlow loss \\
        $\vg$ & Parameter Jacobian of the network output: $\vg\!\triangleq\!\nabla_{\vtheta}\vu_{\vtheta}\!\in\!\mathbb{R}^{d\times p}$ \\
        $\mG_{\vtheta}$ & Parameter-space Gram matrix: $\mG_{\vtheta}\!\triangleq\!\vg\vg^{\!\top}\!\in\!\mathbb{R}^{d\times d}$ \\
        \midrule
        \multicolumn{2}{@{}l}{\emph{Tangent-mixing coefficient}} \\
        $\beta\!\in\![0,1]$ & Tangent-mixing coefficient defined in~\eqref{eq: tangent-mix} \\
        $\beta^{\ast}$ & Closed-form scalar-isotropic minimizer of $M(\beta)$ in~\cref{thm: control-variate-optimum} \\
        $\beta^{\ast}_{\text{matrix}}$ & Matrix-form minimizer in \cref{appx: theorem-proof},~\eqref{eq: beta-star-matrix}\\
        $M(\beta)$ & Conditional MSE of the per-sample gradient~\eqref{eq: g-beta} at coefficient $\beta$ \\
        \midrule
        \multicolumn{2}{@{}l}{\emph{Loss components}} \\
        $\gL_{\text{MF}}$ & Vanilla MeanFlow loss defined by~\eqref{eq: mf-loss} \\
        $\gL_{\text{MF}}^{\text{EMA}}$ & EMA-tangent variant of the MeanFlow loss defined by~\eqref{eq: ema-mf-loss} \\
        $\gL_{\text{FM}}$ & Flow-matching anchor loss defined by~\eqref{eq: fm-anchor} \\
        $\gL_{\beta=1}$ & $\beta\!=\!1$ corner training loss defined by~\eqref{eq: vamf-obj} \\
        \midrule
        \multicolumn{2}{@{}l}{\emph{Operators}} \\
        $\texttt{sg}[\cdot]$ & Stop-gradient operator \\
        $\texttt{JVP}(\vu,\boldsymbol{x},\boldsymbol{v})$ & Jacobian-vector product of $\vu$ evaluated at $\boldsymbol{x}$ in tangent direction $\boldsymbol{v}$ \\
        $\Tr(\cdot)$, $\Cov[\cdot]$, $\Var[\cdot]$ & Trace, covariance, variance \\
        $\mathrm{div}(\vf)$ & Divergence of vector field $\vf$ \\
        \bottomrule
    \end{tabular}
\end{table}

%% file: assets/algorithms/vamf.tex
\begin{algorithm}[t]
    \caption{Training the $\beta\!=\!1$ instantiation with EMA tangent and FM anchor loss}
    \label{alg: va-mf}
    \begin{algorithmic}
        \Require Dataset $\gD$, EMA decay $\mu$, anchor weight $\lambda$, anchor interval $[\delta_{\min},\delta_{\max}]$, learning rate $\eta$
        \State Initialize $\bar{\vtheta}\gets\vtheta$
        \For{$k=1,2,\ldots$}
            \State Sample $(\vx_{0},\vx_{1})\sim\gD\times\gN(\bm{0},\mI_{d})$,\; $t\sim\gU[0,1]$,\; $r\sim\gU[0,t]$,\; $\delta\sim\gU[\delta_{\min},\delta_{\max}]$
            \State $\vx_{t}\gets(1-t)\vx_{0}+t\,\vx_{1}$,\quad $\vv_{\text{cond}}\gets\vx_{1}-\vx_{0}$
            \State $\vv_{\text{tang}}\gets\texttt{sg}\!\left[\vu_{\bar{\vtheta}}(\vx_{t},t,t)\right]$ \Comment{EMA tangent (one extra forward pass)}
            \State $V_{\vtheta}\gets\vu_{\vtheta}(\vx_{t},r,t)+(t{-}r)\,\texttt{sg}\!\left[\texttt{JVP}(\vu_{\vtheta},(\vx_{t},r,t),(\vv_{\text{tang}},0,1))\right]$
            \State $\ell_{\text{MF}}\gets\|V_{\vtheta}-\vv_{\text{cond}}\|_{2}^{2}$
            \State $\ell_{\text{FM}}\gets\|\vu_{\vtheta}(\vx_{t},t{-}\delta,t)-\vv_{\text{cond}}\|_{2}^{2}$ \Comment{FM anchor (one extra forward pass)}
            \State $\vtheta\gets\vtheta-\eta\,\nabla_{\vtheta}\!\left(\ell_{\text{MF}}+\lambda\,\ell_{\text{FM}}\right)$
            \State $\bar{\vtheta}\gets\mu\bar{\vtheta}+(1-\mu)\vtheta$ \Comment{EMA update}
        \EndFor
        \State \Return $\vtheta$ \Comment{Generate: $\hat{\vx}_{0}=\vx_{1}-\vu_{\vtheta}(\vx_{1},0,1)$}
    \end{algorithmic}
\end{algorithm}

%% file: assets/tables/toy.tex
\begin{table}[!t]
    \centering
    \caption{Sliced Wasserstein distance (lower is better) on the 2-D toy benchmark, averaged over three seeds $\{42,0,1\}$ at $200$k steps. SW$_p$ is evaluated on $4096$ samples with $500$ random projections; both metrics use the \emph{same} projection key per evaluation for variance-reduced paired estimates.}
    \label{tab: toy}
    \begin{tabular}{llcc}
        \toprule
        Dataset & Metric & MeanFlow ($\beta\!=\!0$) & $\beta\!=\!1$ \\
        \midrule
        \texttt{checkerboard}    & SW$_1$ & $0.175 \pm 0.036$ & $\mathbf{0.134 \pm 0.029}$ \\
                                 & SW$_2$ & $0.234 \pm 0.042$ & $\mathbf{0.175 \pm 0.046}$ \\
        \midrule
        \texttt{eight\_gaussians}& SW$_1$ & $0.098 \pm 0.016$ & $\mathbf{0.085 \pm 0.010}$ \\
                                 & SW$_2$ & $0.151 \pm 0.022$ & $\mathbf{0.146 \pm 0.023}$ \\
        \midrule
        \texttt{two\_moons}      & SW$_1$ & $0.093 \pm 0.028$ & $\mathbf{0.072 \pm 0.019}$ \\
                                 & SW$_2$ & $0.177 \pm 0.068$ & $\mathbf{0.093 \pm 0.025}$ \\
        \midrule
        \texttt{swiss\_roll}     & SW$_1$ & $0.098 \pm 0.058$ & $\mathbf{0.045 \pm 0.003}$ \\
                                 & SW$_2$ & $0.177 \pm 0.151$ & $\mathbf{0.060 \pm 0.007}$ \\
        \midrule
        \texttt{two\_spirals}    & SW$_1$ & $\mathbf{0.034 \pm 0.001}$ & $0.037 \pm 0.001$ \\
                                 & SW$_2$ & $\mathbf{0.044 \pm 0.001}$ & $0.046 \pm 0.001$ \\
        \midrule
        \texttt{pinwheel}        & SW$_1$ & $0.094 \pm 0.011$ & $\mathbf{0.059 \pm 0.005}$ \\
                                 & SW$_2$ & $0.128 \pm 0.014$ & $\mathbf{0.076 \pm 0.005}$ \\
        \bottomrule
    \end{tabular}
\end{table}

%% file: assets/tables/dgmm-beta-full.tex
\begin{table}[!t]
    \centering
    \caption{DGMM $\beta$-sweep: full grid of mean$\pm$SEM SW$_{1}$ at convergence ($200$k steps, three seeds $\{42,0,1\}$). Bolded cell per column marks the empirical $\beta^{\ast}\!=\!\arg\min_{\beta}\mathrm{SW}_{1}(\beta)$ at that dimension. The summary view is in~\Cref{tab: dgmm-beta-summary}.}
    \label{tab: dgmm-beta-full}
    \resizebox{\textwidth}{!}{%
    \begin{tabular}{c|cccccc}
        \toprule
        $\beta$ & \multicolumn{6}{c}{$d$} \\
        \cmidrule(lr){2-7}
            & $2$ & $4$ & $8$ & $16$ & $32$ & $64$ \\
        \midrule
        $0.0$ & $0.152{\scriptstyle\pm0.022}$ & $0.083{\scriptstyle\pm0.011}$ & $0.062{\scriptstyle\pm0.003}$ & $0.044{\scriptstyle\pm0.002}$ & $0.034{\scriptstyle\pm0.001}$ & $0.028{\scriptstyle\pm0.003}$ \\
        $0.1$ & $0.151{\scriptstyle\pm0.025}$ & $0.083{\scriptstyle\pm0.011}$ & $0.062{\scriptstyle\pm0.003}$ & $0.044{\scriptstyle\pm0.002}$ & $\mathbf{0.034{\scriptstyle\pm0.001}}$ & $0.027{\scriptstyle\pm0.003}$ \\
        $0.2$ & $0.149{\scriptstyle\pm0.024}$ & $0.083{\scriptstyle\pm0.011}$ & $0.062{\scriptstyle\pm0.003}$ & $0.044{\scriptstyle\pm0.002}$ & $0.037{\scriptstyle\pm0.003}$ & $0.028{\scriptstyle\pm0.003}$ \\
        $0.3$ & $0.148{\scriptstyle\pm0.022}$ & $0.083{\scriptstyle\pm0.011}$ & $0.062{\scriptstyle\pm0.003}$ & $0.044{\scriptstyle\pm0.002}$ & $0.038{\scriptstyle\pm0.004}$ & $0.028{\scriptstyle\pm0.003}$ \\
        $0.4$ & $0.149{\scriptstyle\pm0.022}$ & $0.083{\scriptstyle\pm0.011}$ & $0.062{\scriptstyle\pm0.003}$ & $0.044{\scriptstyle\pm0.002}$ & $0.041{\scriptstyle\pm0.007}$ & $\mathbf{0.025{\scriptstyle\pm0.000}}$ \\
        $0.5$ & $0.150{\scriptstyle\pm0.022}$ & $0.084{\scriptstyle\pm0.011}$ & $0.062{\scriptstyle\pm0.003}$ & $0.044{\scriptstyle\pm0.002}$ & $0.041{\scriptstyle\pm0.007}$ & $0.026{\scriptstyle\pm0.001}$ \\
        $0.6$ & $0.150{\scriptstyle\pm0.021}$ & $0.083{\scriptstyle\pm0.011}$ & $0.062{\scriptstyle\pm0.003}$ & $0.044{\scriptstyle\pm0.002}$ & $0.035{\scriptstyle\pm0.001}$ & $0.026{\scriptstyle\pm0.001}$ \\
        $0.7$ & $0.148{\scriptstyle\pm0.020}$ & $0.083{\scriptstyle\pm0.011}$ & $0.062{\scriptstyle\pm0.003}$ & $0.044{\scriptstyle\pm0.002}$ & $0.034{\scriptstyle\pm0.001}$ & $0.026{\scriptstyle\pm0.001}$ \\
        $0.8$ & $0.148{\scriptstyle\pm0.016}$ & $\mathbf{0.083{\scriptstyle\pm0.011}}$ & $0.062{\scriptstyle\pm0.003}$ & $0.044{\scriptstyle\pm0.002}$ & $0.036{\scriptstyle\pm0.002}$ & $0.027{\scriptstyle\pm0.002}$ \\
        $0.9$ & $0.148{\scriptstyle\pm0.015}$ & $0.083{\scriptstyle\pm0.011}$ & $0.062{\scriptstyle\pm0.003}$ & $0.044{\scriptstyle\pm0.002}$ & $0.037{\scriptstyle\pm0.003}$ & $0.027{\scriptstyle\pm0.001}$ \\
        $1.0$ & $\mathbf{0.146{\scriptstyle\pm0.015}}$ & $0.083{\scriptstyle\pm0.011}$ & $\mathbf{0.062{\scriptstyle\pm0.003}}$ & $\mathbf{0.044{\scriptstyle\pm0.003}}$ & $0.036{\scriptstyle\pm0.002}$ & $0.028{\scriptstyle\pm0.001}$ \\
        \bottomrule
    \end{tabular}%
    }
\end{table}

%% file: assets/tables/dit-fid.tex
\begin{table}[!t]
    \centering
    \caption{DiT-B/4 / ImageNet-$256$ FID$_{50\text{k}}$ (lower is better) at representative $5$k-step checkpoints across the four-point $\beta$-sweep. All four runs completed their scheduled $300$k-step training and have a matched-step eval at step $295$k. Bottom rows give the matched-step deltas. The coarse ordering ($\beta\!=\!1$ well above the interior runs, $\beta\!=\!0.5$ above the baseline) holds at every matched-step checkpoint after the early-training transient; the finer $\beta\!=\!0$ vs $\beta\!=\!0.25$ gap (sub-$0.5$ FID near convergence) is within across-seed variation.}
    \label{tab: dit-fid}
    \resizebox{\textwidth}{!}{%
    \begin{tabular}{lcccccccc}
        \toprule
        step (k)                                     & $30$    & $50$    & $100$   & $150$   & $200$   & $250$   & $275$   & $\mathbf{295}$ \\
        \midrule
        baseline ($\beta\!=\!0$)                     & $128.9$ & $54.5$  & $22.0$  & $15.5$  & $13.0$  & $11.9$  & $11.6$  & $\mathbf{11.37}$ \\
        $\beta\!=\!0.25$                             & $138.5$ & $61.6$  & $23.3$  & $16.2$  & $13.3$  & $12.1$  & $11.7$  & $\mathbf{11.76}$ \\
        $\beta\!=\!0.5$                              & $145.2$ & $66.9$  & $25.6$  & $17.6$  & $14.4$  & $13.2$  & $12.7$  & $\mathbf{12.51}$ \\
        $\beta\!=\!1$                                & $148.8$ & $79.9$  & $41.7$  & $32.1$  & $27.6$  & $25.1$  & $24.1$  & $\mathbf{23.36}$ \\
        \midrule
        $\Delta$ ($\beta\!=\!0.25$ $-$ baseline)     & $+9.5$  & $+7.1$  & $+1.3$  & $+0.7$  & $+0.4$  & $+0.2$  & $+0.1$  & $\mathbf{+0.39}$ \\
        $\Delta$ ($\beta\!=\!0.5$ $-$ baseline)      & $+16.2$ & $+12.4$ & $+3.6$  & $+2.1$  & $+1.4$  & $+1.3$  & $+1.0$  & $\mathbf{+1.14}$ \\
        $\Delta$ ($\beta\!=\!1$ $-$ baseline)        & $+19.9$ & $+25.4$ & $+19.7$ & $+16.6$ & $+14.6$ & $+13.2$ & $+12.5$ & $\mathbf{+11.99}$ \\
        \bottomrule
    \end{tabular}
    }
\end{table}

%% file: assets/tables/dit-bias-shrinkage.tex
\begin{table}[!t]
    \centering
    \caption{Direct DiT-checkpoint measurement of the bias-variance decomposition ingredients of~\Cref{thm: control-variate-optimum} at $t\!=\!0.5$. The EMA bias $\|\vb\|^{2}$ averages over $128$ synthetic-Gaussian $\vx_{t}$ samples and is reported for both methods. The shrinkage factor uses $\sigma^{2}d\!\approx\!8.2{\times}10^{3}$. The predicted $\beta^{\ast}$ uses $\kappa\!=\!1$.}
    \label{tab: dit-bias-shrinkage}
    \begin{tabular}{lcccc}
        \toprule
        step (k)                              & $\|\vb\|^{2}$ (avg) & $\sigma^{2}d/(\sigma^{2}d{+}\|\vb\|^{2})$ & $\beta^{\ast}$ \\
        \midrule
        $20$                                  & $\sim\!700$         & $0.92$                                    & $0.46$ \\
        $40$                                  & $\sim\!150$         & $0.98$                                    & $0.49$ \\
        $80$                                  & $\sim\!50$          & $0.99$                                    & $0.50$ \\
        \bottomrule
    \end{tabular}
\end{table}